\documentclass[runningheads]{llncs}
\usepackage[T1]{fontenc}
\usepackage{graphicx}
\usepackage{booktabs}
\usepackage[misc]{ifsym}
\usepackage{amssymb}
\usepackage{amsmath}
\usepackage{hyperref}
\usepackage{cleveref}
\newcommand{\corr}{(\Letter)}
\usepackage{mwe}
\usepackage{booktabs}
\usepackage{multirow}
\usepackage{subcaption}
\usepackage[table,dvipsnames]{xcolor}
\usepackage{import}

\begin{document}

\title{Conditional Attribution for Root Cause Analysis in Time-Series Anomaly Detection}

\titlerunning{Conditional Attribution for Time-Series RCA}
% If the full title of your paper is short enough to also fit in the running head, you can omit the abbreviated paper title here. You can check as follows: if you comment out the \titlerunning line, something will appear in the header of all odd-numbered pages of your PDF from page 3 onward. This something is either the full title (in which case all is well), or the error message "Title Suppressed Due to Excessive Length". If this error message appears, you're going to want to provide an abbreviated title within the \titlerunning command, because if you won't do it, Springer will do it for you.

%N.B.: Author information (both in the \author{} and \authorrunning{} command) should only be present in the Camera-Ready Version of your paper. The version that you initially submit for review, ought to be double-blind. So, when initially submitting your paper, use:
%\author{Author information scrubbed for double-blind reviewing}

\author{Shashank Mishra\inst{1} \corr \and
Karan Patil\inst{1,3} \and Cedric Schockaert\inst{2} \and
Didier Stricker\inst{1,3} \and Jason Rambach\inst{1}}

% You may leave out the orcidID information, if you want to.
% Use \corr to indicate the corresponding author. Note the spacing around the \corr command. Only one author can be the corresponding author.

%N.B.: comment out the \authorrunning{} command for the double-blind version of your paper submitted for review. Later, if your paper is accepted, use the command for the Camera-Ready Version.
\authorrunning{S. Mishra et al.}
% If there are two authors, write 'A.L. Benjamin and C.C. Broadus Jr.'
% If there are more than two authors, '[...] et al.' is used.

\institute{
German Research Center for Artificial Intelligence (DFKI), Kaiserslautern, Germany 
\email{\{shashank.mishra, karan\_sanjay.patil, didier.stricker, jason.rambach\}@dfki.de} 
\and 
Paul Wurth S.A, Luxembourg 
\email{cedric.schockaert@sms-group.com} 
\and 
Rheinland-Pfälzische Technische Universität Kaiserslautern-Landau, Germany}

\maketitle              % typeset the header of the contribution

\begin{abstract}
Root cause analysis (RCA) for time-series anomaly detection is critical for the reliable operation of complex real-world systems. Existing explanation methods often rely on unrealistic feature perturbations and ignore temporal and cross-feature dependencies, leading to unreliable attributions. We propose a conditional attribution framework that explains anomalies relative to contextually similar normal system states. Instead of using marginal or randomly sampled baselines, our method retrieves representative normal instances conditioned on the anomalous observation, enabling dependency-preserving and operationally meaningful explanations. To support high-dimensional time-series data, contextual retrieval is performed in learned low-dimensional representations using both variational autoencoder latent spaces and UMAP manifold embeddings. By grounding the retrieval process in the system's learned manifold, this strategy avoids out-of-distribution artifacts and ensures attribution fidelity while maintaining computational efficiency. We further introduce confidence-aware and temporal evaluation metrics for assessing explanation reliability and responsiveness. Experiments on the SWaT and MSDS benchmarks demonstrate that the proposed approach consistently improves root-cause identification accuracy, temporal localization, and robustness across multiple anomaly detection models. These results highlight the practical utility of conditional attribution for explainable anomaly diagnosis in complex time-series systems. Code and models are available at: \href{https://github.com/dfki-av/Conditional-Attribution-for-Root-Cause-Analysis-in-Time-Series-Anomaly-Detection}{GitHub repository}.

\end{abstract}

\section{Introduction}

Modern anomaly detection systems are increasingly deployed in safety-critical and large-scale environments such as industrial control systems, cyber-physical infrastructure, and cloud platforms 
\cite{kumar2025anomaly,xie2020multivariateswat,abshari2025survey}. While recent advances in deep learning have significantly improved anomaly detection accuracy, practical deployment requires more than detection alone: operators must understand \emph{why} an anomaly occurred in order to diagnose faults, take corrective action, and prevent recurrence. This has positioned root cause analysis (RCA) as a central challenge in multivariate time-series monitoring.

Existing RCA methods often adapt feature attribution techniques \cite{lundberg2017unifiedshap,sundararajan2017axiomatic} that assume feature independence. In multivariate time-series, strong temporal and cross-sensor correlations cause these methods to rely on unrealistic, out-of-distribution (OOD) perturbations \cite{vlassopoulos2020explaining}. Recent attempts to address these limitations through structured or group-wise attributions \cite{jullum2021groupshapleyefficientpredictionexplanation,covert2020understandingglobalfeaturecontributions} still rely on fixed background datasets or heuristic sampling strategies, which fail to preserve the conditional structure of complex physical processes, scale poorly in high-dimensional settings, and offer limited guarantees regarding the fidelity of the explanations. Consequently, current explanations are often operationally implausible, creating a gap between detection and actionable diagnostic insight.

In this work, we bridge this gap with a \textbf{conditional attribution framework} that explains anomalies relative to contextually similar normal system states. To handle high-dimensionality and noise, we perform contextual retrieval in learned manifold representations using both Variational Autoencoders (VAE) \cite{kingma2013auto} and UMAP \cite{healy2024uniform}. By grounding attribution in representative normal instances that reflect realistic operating conditions, our framework preserves structural dependencies and ensures that explanations are both model-agnostic and physically consistent.

Beyond feature-level identification, we address the critical need for \textit{temporal localization} and \textit{principled evaluation} in time-series RCA. Our framework enables precise onset localization within long-sequence anomalies by performing attribution over structured temporal windows, capturing the evolution of abnormal behavior. To bridge the gap between heuristic rankings and operational utility, we introduce two novel metrics: (i) \textbf{CW-RCS}, a confidence-aware measure that incorporates attribution strength to penalize diffuse explanations, and (ii) \textbf{TemporalHM}, which quantifies the responsiveness and stability of root-cause identification following anomaly onset. Together, these contributions provide a more rigorous assessment of explanation fidelity and diagnostic reliability than standard retrieval-based measures.

% We evaluate the proposed framework through extensive experiments on the SWaT industrial control benchmark \cite{xie2020multivariateswat}, where ground-truth anomalous sensors are available, as well as through controlled anomaly injection scenarios. In addition, we demonstrate the practical applicability of our approach using real industrial monitoring data, highlighting its potential for supporting diagnostic workflows in operational environments.

% In summary, this work makes the following contributions:

% \begin{enumerate}

% \item \textbf{Conditional attribution framework for root cause analysis} that identifies anomalous variables relative to contextually similar normal system states while preserving temporal and cross-feature dependencies.

% \item \textbf{Representation-guided contextual retrieval} using variational autoencoder latent spaces and UMAP embeddings to enable scalable and model-agnostic attribution in high-dimensional time-series data.

% \item \textbf{Two confidence-aware evaluation metrics for RCA:}  
% (i) Confidence-Weighted Top-$K$ Recall, and  
% (ii) Root Cause Identification Time, capturing explanation reliability and temporal responsiveness.

% \end{enumerate}

% Together, these contributions advance the reliability and practical utility of explainable anomaly detection, bridging the gap between anomaly detection performance and actionable root cause diagnosis in complex time-series systems.

We evaluate our framework on the \textbf{SWaT} \cite{xie2020multivariateswat} industrial benchmark and the high-dimensional \textbf{MSDS} \cite{nedelkoski2020multi} dataset, alongside a real-world case study on industrial blast furnace monitoring data. Our results demonstrate superior localization accuracy and operational utility in complex, dependent systems.

\textbf{Our contributions are threefold:}
\begin{enumerate}
    \item \textbf{Conditional Attribution Framework:} A novel RCA approach that explains anomalies relative to contextually similar normal states, preserving temporal and cross-feature dependencies to avoid out-of-distribution artifacts.
    \item \textbf{Manifold-Guided Contextual Retrieval:} A scalable, model-agnostic st\-rategy using VAE and UMAP embeddings to retrieve representative baselines in high-dimensional latent spaces.
    \item \textbf{Rigorous Evaluation Metrics:} We introduce \textbf{CW-RCS} (Confidence-Weighted Root Cause Score) and \textbf{TemporalHM} to quantify explanation reliability and temporal responsiveness, bridging the gap between detection and actionable diagnosis.
\end{enumerate}
This work provides a principled foundation for high-fidelity, dependency-preserving explanations in complex time-series systems.

\section{Related Work}
% \subsection{Time-Series Anomaly Detection}
% While deep learning models like Transformers \cite{xu2021anomaly} and TranAD \cite{tuli2022tranad} have significantly improved detection accuracy by modeling complex temporal dependencies, they typically function as "black boxes". Anomaly scores alone are insufficient for operational diagnosis, necessitating explanation-aware methods that provide reliable root cause information alongside detection.

%Recent advances in time-series anomaly detection have been driven by deep learning models such as recurrent networks, temporal convolutional architectures, and Transformer-based methods \cite{xu2021anomaly}, which achieve strong performance by modeling complex temporal dependencies. Approaches such as Anomaly Transformer, TranAD \cite{tuli2022tranad}, and related attention-based frameworks \cite{wen2022transformers} have demonstrated improved sensitivity to subtle and long-range anomalies. However, despite these advances, most existing methods focus primarily on detection accuracy and provide limited insight into the underlying causes of anomalous behavior. In practice, anomaly scores alone are insufficient for diagnosis, as they do not reveal which variables or temporal segments contribute to the detected anomaly. This limitation motivates the need for explanation-aware anomaly detection methods that can provide reliable and interpretable root cause information alongside accurate detection.

\subsection{Root Cause Analysis (RCA) in Multivariate Time-Series}
Existing RCA paradigms generally fall into three categories: (i) Statistical methods (e.g., EXstream \cite{zhang2017exstream}, SHAP \cite{lundberg2017unifiedshap,jullum2021groupshapleyefficientpredictionexplanation}) which rank features by deviation intensity but often fall into the "symptom-as-cause" trap; (ii) Causal Discovery (e.g., MicroRCA \cite{wu2020microrca}, CloudRanger \cite{wang2018cloudranger}) which utilize dependency graphs but struggle with the high-dimensionality and non-linearity of industrial data; and (iii) Reconstruction-based methods (e.g., OmniAnomaly \cite{su2019robust}, Interfusion \cite{li2021multivariate}) which use generative models to identify root causes via reconstruction error or counterfactuals . Despite their utility, these methods frequently rely on marginal perturbations or static backgrounds that fail to preserve the complex conditional dependencies inherent in time-series systems.

\subsection{Feature Attribution and Explainability}
Feature attribution methods aim to quantify the contribution of individual input variables to model predictions. Model-agnostic approaches such as LIME \cite{ribeiro2016should} and SHAP \cite{lundberg2017unifiedshap} approximate local decision boundaries or decompose predictions into additive feature contributions. Gradient-based techniques, including Integrated Gradients \cite{sundararajan2017axiomatic}, extend attribution to deep neural networks. While these methods have been widely adopted for tabular and vision tasks, their direct application to multivariate time-series anomaly detection remains challenging. In particular, many attribution techniques rely on marginal perturbations or independence assumptions that fail to preserve temporal and cross-feature dependencies, limiting their reliability for root cause analysis in complex monitoring systems.

\section{Problem Formulation}
\label{problem_formulation}
We study root cause analysis (RCA) for anomalies in multivariate time-series. The key challenge is to produce explanations that are \emph{faithful} to the anomaly detector while respecting the strong temporal and cross-sensor dependencies present in real systems.

\subsection{Multivariate Time-Series Anomaly Detection}
Let $\mathbf{x}_t \in \mathbb{R}^d$ denote the $d$-dimensional sensor measurement at time $t \in \{1,\dots,T\}$, and let $\mathbf{X} = (\mathbf{x}_1,\dots,\mathbf{x}_T)$ denote the full multivariate time-series. We define the sliding window of length $w$ starting at time $t$ as
\begin{equation}
\mathbf{W}_t := (\mathbf{x}_t,\mathbf{x}_{t+1},\dots,\mathbf{x}_{t+w-1}) \in \mathbb{R}^{w \times d}.
\end{equation}
An anomaly detector is a function $f:\mathbb{R}^{w \times d} \rightarrow \mathbb{R}$, which maps a window $\mathbf{W}_t$ to an anomaly score $s_t := f(\mathbf{W}_t)$.

\subsection{Root Cause Attribution Objective}
The objective of RCA is to estimate an attribution tensor $\boldsymbol{\Phi}_t \in \mathbb{R}^{w \times d}$, where each element $\Phi_t(\tau, j)$ quantifies the contribution of sensor $j$ at relative time $\tau$. The total sensor-level attribution, which we denote as $\phi_{t,j}$ to differentiate it from the local attribution values, is obtained by aggregating over the window:
\begin{equation}
\phi_{t,j} = \sum_{\tau=1}^{w} \Phi_t(\tau,j).
\end{equation}

\subsection{Conditional Attribution Formulation}
Let $\mathbf{W}_t^{(-j)}$ denote the window context excluding the $j$-th sensor's trajectory. To quantify the contribution of sensor $j$ relative to normal system behavior, let $p(\cdot)$ denote the distribution of normal (non-anomalous) system windows, and let $\mathbf{W}'_j$ denote a replacement trajectory for sensor $j$ sampled from the corresponding normal distribution. We define the \textbf{marginal attribution} and the \textbf{conditional attribution} as:

\begin{equation}
\phi_{t,j}^{\text{marg}} = \mathbb{E}_{\mathbf{W}'_j \sim p(\mathbf{W}_j)} \left[ f(\mathbf{W}_t) - f(\mathbf{W}_t^{(-j)}, \mathbf{W}'_j) \right]
\end{equation}
\begin{equation}
\phi_{t,j}^{\text{cond}} = \mathbb{E}_{\mathbf{W}'_j \sim p(\mathbf{W}_j \mid \mathbf{W}_t^{(-j)})} \left[ f(\mathbf{W}_t) - f(\mathbf{W}_t^{(-j)}, \mathbf{W}'_j) \right].
\end{equation}

In practice, direct sampling from the conditional distribution of normal behavior is intractable. We approximate it using a contextual neighborhood $\mathcal{N}(\mathbf{W}_t)$, defined as the set of $K$ normal windows $\mathbf{W}'$ that minimize the contextual distance $d(\mathbf{W}_t^{(-j)}, \mathbf{W}'^{(-j)}) = \|\text{vec}(\mathbf{W}_t^{(-j)}) - \text{vec}(\mathbf{W}'^{(-j)})\|_2$ in the context space. This yields the empirical estimator:
\begin{equation}
\hat{\phi}_{t,j} = \frac{1}{K} \sum_{\mathbf{W}' \in \mathcal{N}(\mathbf{W}_t)} \left[ f(\mathbf{W}_t) - f(\mathbf{W}_t^{(-j)}, \mathbf{W}'_j) \right].
\label{eq:equation5}
\end{equation}

\begin{proposition}[Dependency-Preserving Attribution]
Assume $f(\cdot)$ is $L$-Lipschitz continuous with respect to the $j$-th sensor trajectory, i.e., $|f(\mathbf{W}_t^{(-j)}, \mathbf{u}) - f(\mathbf{W}_t^{(-j)}, \mathbf{v})| \leq L \|\mathbf{u} - \mathbf{v}\|_2$ for any trajectories $\mathbf{u}, \mathbf{v} \in \mathbb{R}^w$. The attribution bias induced by marginal perturbation is bounded by:
\begin{equation}
|\phi_{t,j}^{\text{cond}} - \phi_{t,j}^{\text{marg}}| \leq L \cdot W_1\left( p(\mathbf{W}_j \mid \mathbf{W}_t^{(-j)}), p(\mathbf{W}_j) \right),
\end{equation}
where $W_1$ is the Wasserstein-1 distance computed over trajectories in $\mathbb{R}^w$ using the $\ell_2$-norm as the ground metric.
\end{proposition}

\begin{proof}
Let $g(\mathbf{W}'_j) = f(\mathbf{W}_t^{(-j)}, \mathbf{W}'_j)$. By the linearity of expectation:
\[ |\phi_{t,j}^{\text{cond}} - \phi_{t,j}^{\text{marg}}| = \left| \mathbb{E}_{p(\mathbf{W}_j \mid \mathbf{W}_t^{(-j)})} [g(\mathbf{W}'_j)] - \mathbb{E}_{p(\mathbf{W}_j)} [g(\mathbf{W}'_j)] \right|. \]
By the Lipschitz assumption, $g$ is an $L$-Lipschitz function. According to the Kantorovich-Rubinstein duality, the following holds:
\[ \left| \mathbb{E}_{P}[g] - \mathbb{E}_{Q}[g] \right| \leq L \cdot W_1(P, Q). \]
By substituting $P = p(\mathbf{W}_j \mid \mathbf{W}_t^{(-j)})$ and $Q = p(\mathbf{W}_j)$, we obtain the bound. The bias represents the systematic error introduced by evaluating the detector on out-of-distribution counterfactuals that violate learned sensor dependencies.
\end{proof}

\begin{figure}[t]
  \centering
  \includegraphics[width=0.9\linewidth]{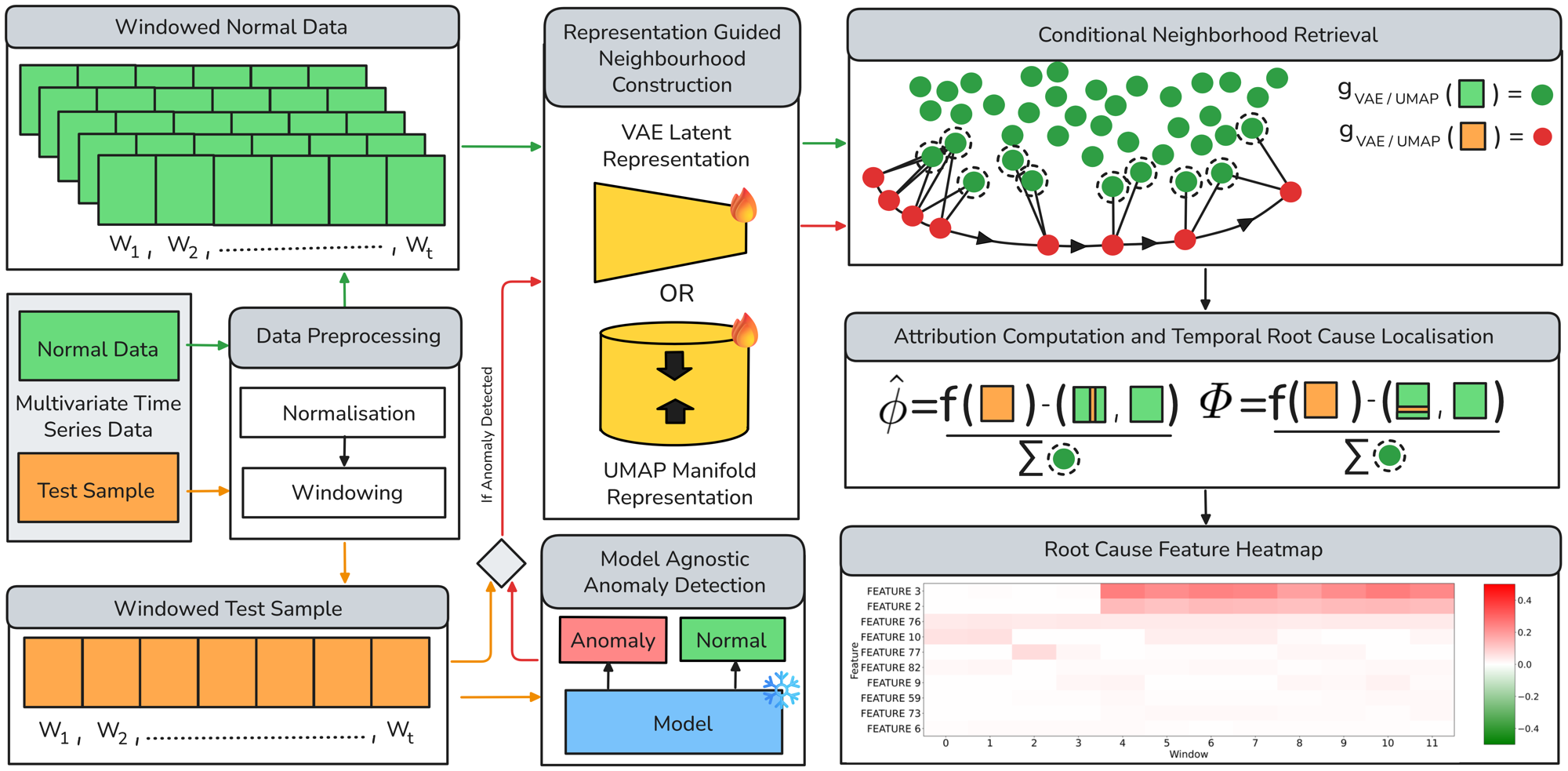}
  \caption{Overview of the proposed model-agnostic RCA pipeline. Steps include latent-space neighborhood construction, conditional attribution computation, and temporal root cause localization for multivariate industrial sensors.}
  \label{fig:architecture}
\end{figure}

% \section{Conditional Attribution Framework}
% We now describe the practical framework used to operationalize conditional attribution for root cause analysis. While \Cref{problem_formulation} defines attribution relative to the conditional distribution of normal system behavior, direct sampling from this distribution is intractable in high-dimensional multivariate time-series. We approximate it through contextual neighborhood retrieval, where representative normal windows are selected based on similarity to the anomalous context. Attribution is then computed via dependency-preserving counterfactual replacement using these retrieved neighborhoods, yielding a scalable and model-agnostic explanation framework.

\section{Conditional Attribution Framework}
We operationalize the conditional attribution for RCA through a scalable, model-agnostic framework illustrated in \Cref{fig:architecture}. Since direct sampling from the conditional distribution $P(\mathbf{x}_j \mid \mathbf{x}_{-j})$ is intractable in high-dimensional time-series, we approximate it via \textbf{contextual neighborhood retrieval}. Representative normal windows are selected based on their proximity to the anomalous context in a learned representation space. Attribution is then computed through dependency-preserving counterfactual replacement using these retrieved neighbors, ensuring explanations remain grounded in the system's manifold and avoid out-of-distribution artifacts.

\subsection{Representation-Guided Neighborhood Construction}

% \begin{figure}[t]
%   \centering
%   \includegraphics[width=0.9\linewidth]{images/latent_space_dummy_single.png}
%   \caption{Encoded Samples comparison for VAE and UMAP for SWaT dataset attacks (dummy: total 36 attacks)}
%   \label{fig:latent_space_dummy}
% \end{figure}

\begin{figure*}[t]
    \centering
    \begin{subfigure}[t]{0.24\textwidth}
        \centering
        \includegraphics[width=\linewidth]{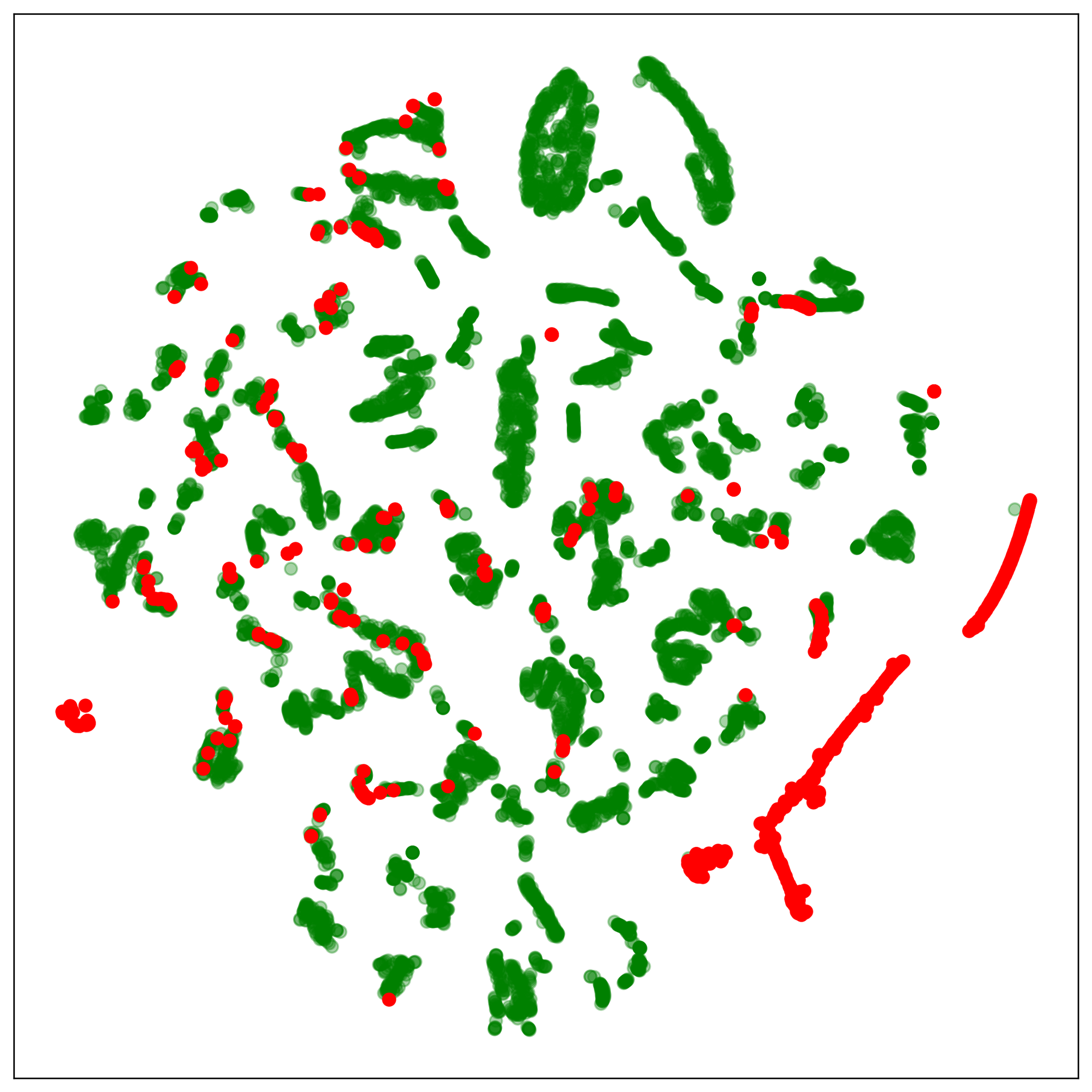}
        \label{fig:img1}
    \end{subfigure}
    \hfill
    \begin{subfigure}[t]{0.24\textwidth}
        \centering
        \includegraphics[width=\linewidth]{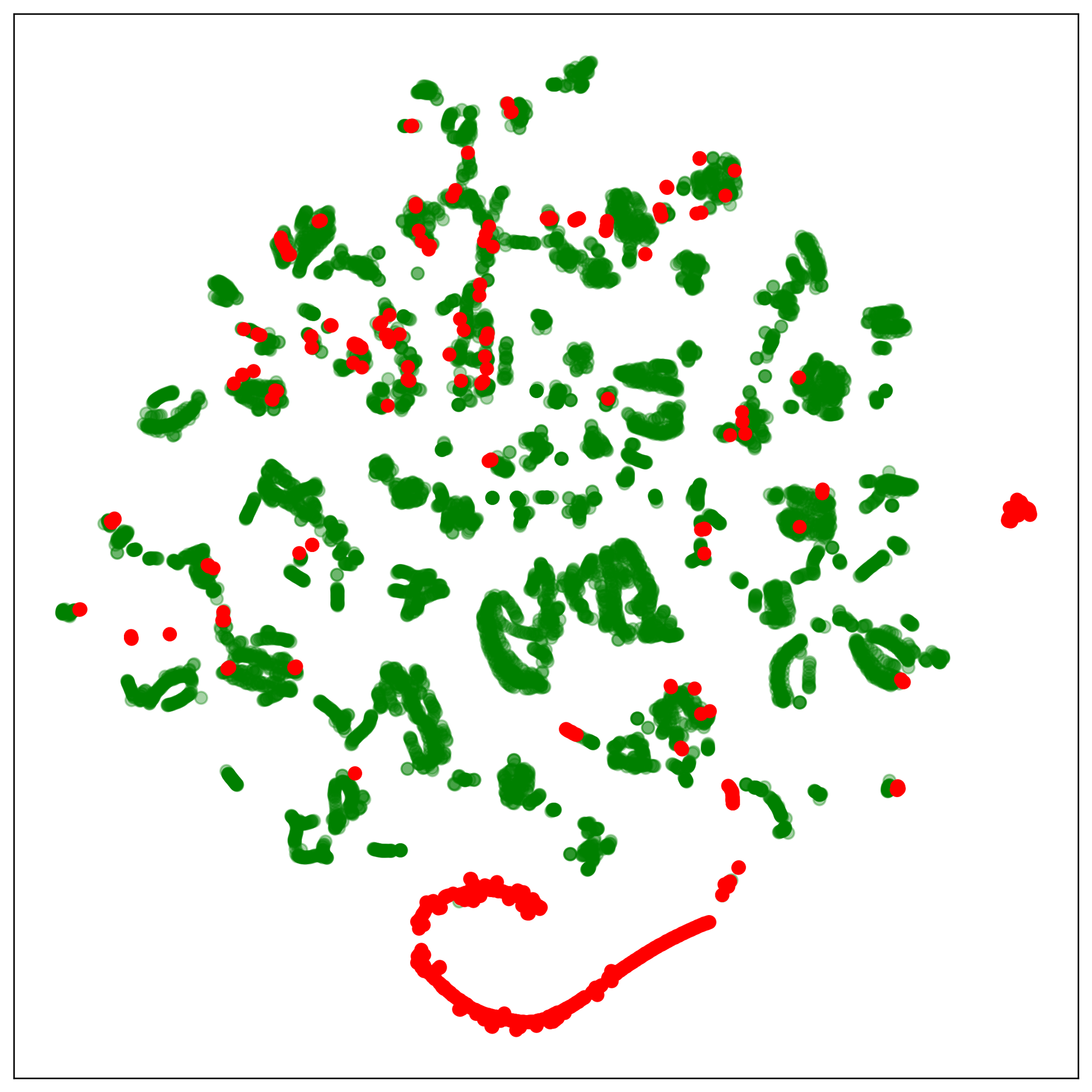}
        \label{fig:img2}
    \end{subfigure}
    \hfill
    \begin{subfigure}[t]{0.24\textwidth}
        \centering
        \includegraphics[width=\linewidth]{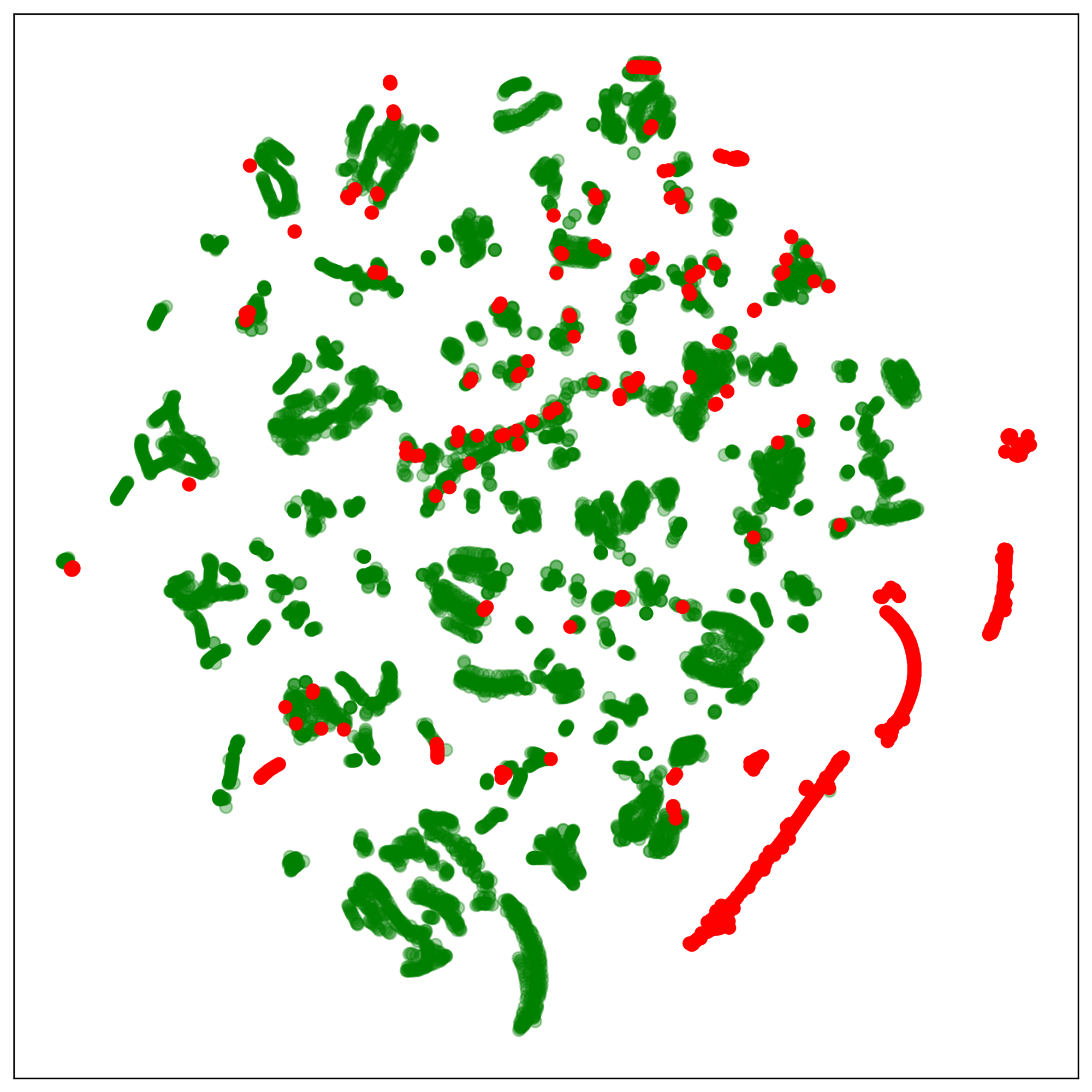}
        \label{fig:img3}
    \end{subfigure}
    \hfill
    \begin{subfigure}[t]{0.24\textwidth}
        \centering
        \includegraphics[width=\linewidth]{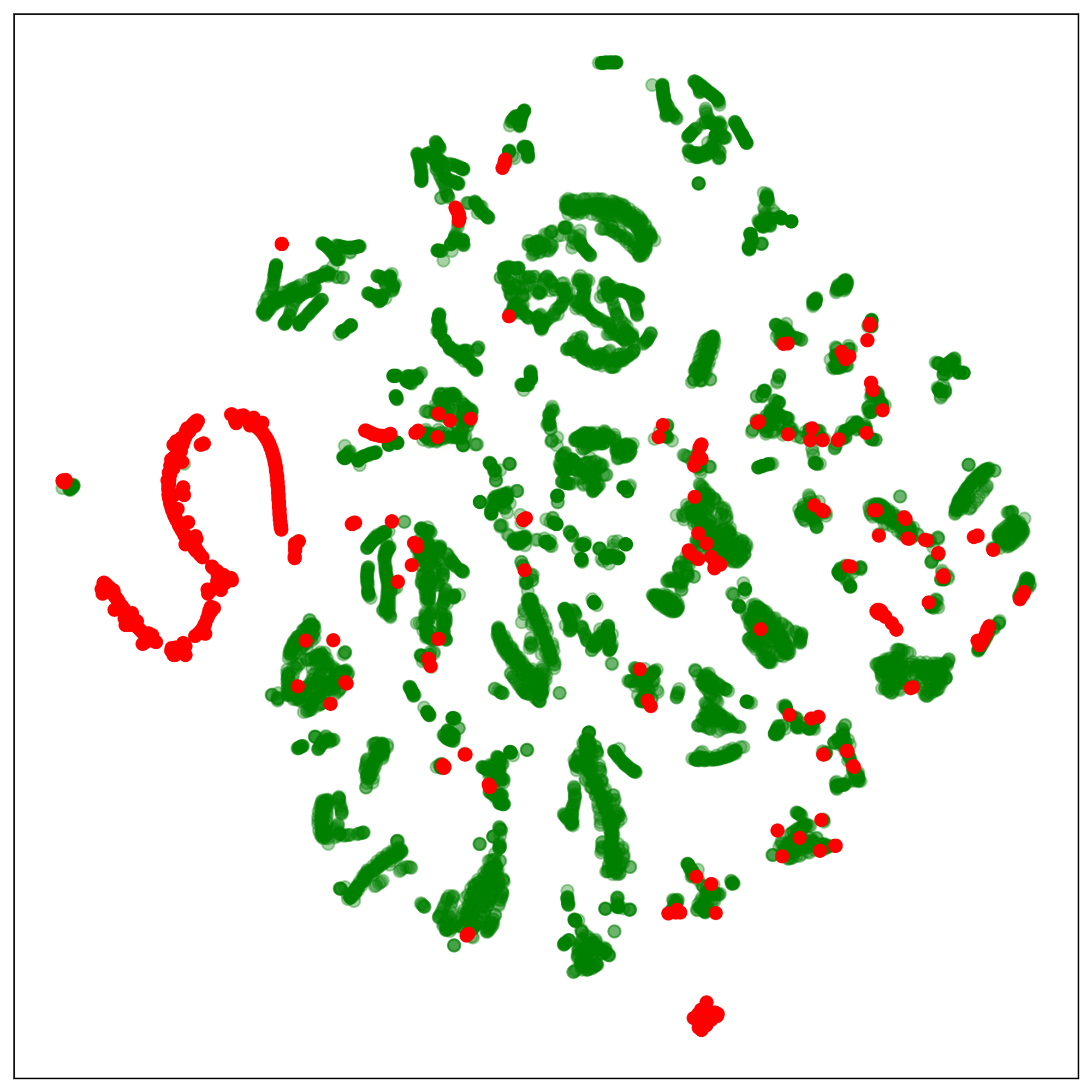}
        \label{fig:img4}
    \end{subfigure}
    \caption{SWaT Manifold Topology ($d \in {4,8,16,32}$). Normal (green) and anomalous (red) regions overlap significantly, indicating anomalies can be locally indistinguishable from normal operating modes. This highlights the need for a conditional framework.}
    \label{fig:four_images}
\end{figure*}

Direct neighborhood construction in the input space becomes unreliable in high-dimensional multivariate time-series due to the curse of dimensionality and noise sensitivity. To obtain semantically meaningful conditional neighborhoods, we perform similarity retrieval in learned low-dimensional representations that preserve system dependencies.

Let $g:\mathbb{R}^{w \times d} \rightarrow \mathbb{R}^k$ denote an embedding function that maps windows to a $k$-dimensional representation space. The conditional neighborhood is then defined as
%\vspace{-2em} %becuase this is reducing the gap between diagram adn text.
\begin{equation}
\mathcal{N}(\mathbf{W}_t)
=
\operatorname*{arg\,min}_{\{\mathbf{W}^{(i)}\} \subset \mathcal{D}_{\text{norm}}}^{K}
\left\|
g(\mathbf{W}_t^{(-j)}) - g(\mathbf{W}^{(i)(-j)})
\right\|_2.
\end{equation}

By performing retrieval in representation space, the neighborhood better reflects latent system structure and nonlinear dependencies, yielding more faithful approximations of the conditional distribution.

\paragraph{\textbf{VAE Latent Representation}}

We first instantiate $g(\cdot)$ using the encoder of a variational autoencoder \cite{kingma2013auto} trained on normal system windows. Let the encoder map each window to a latent variable $\mathbf{z} \sim q_\theta(\mathbf{z}\mid\mathbf{W})$. We define the embedding as the posterior mean:

\begin{equation}
g_{\text{VAE}}(\mathbf{W}) = \mathbb{E}[\mathbf{z} \mid \mathbf{W}].
\end{equation}

Neighborhood retrieval in latent space captures nonlinear system dependencies while reducing dimensionality, improving both computational efficiency and conditional fidelity.

\paragraph{\textbf{UMAP Manifold Representation}}

As an alternative representation, we employ Uniform Manifold Approximation and Projection (UMAP) to learn a low-dimensional embedding that preserves local neighborhood topology. Let

\begin{equation}
g_{\text{UMAP}} : \mathbb{R}^{w \times d} \rightarrow \mathbb{R}^k
\end{equation}

denote the learned manifold embedding. Conditional neighborhoods constructed in this space capture geometric similarity between system states, enabling non-parametric retrieval that preserves local manifold connectivity without the reconstruction bias of generative models.

% \subsection{Representation-Guided Neighborhood Retrieval}
% Input-space neighborhood construction is often unreliable in high-dimensional settings due to noise and the curse of dimensionality. To obtain semantically meaningful conditional neighborhoods, we map windows to a $k$-dimensional space via an embedding function $g:\mathbb{R}^{w \times d} \rightarrow \mathbb{R}^k$ that preserves system dependencies. The conditional neighborhood $\mathcal{N}(\mathbf{W}_t)$ is defined as the set of $K$ normal windows $\{\mathbf{W}^{(i)}\} \subset \mathcal{D}_{\text{norm}}$ minimizing the Euclidean distance in the representation space:
% \begin{equation}
% \mathcal{N}(\mathbf{W}_t) = \operatorname*{arg\,min}_{\{\mathbf{W}^{(i)}\}}^{K} \| g(\mathbf{W}_t) - g(\mathbf{W}^{(i)}) \|_2.
% \end{equation}

% We instantiate $g(\cdot)$ using two distinct strategies. \textbf{(i) VAE Latent Representation:} We utilize the encoder of a Variational Autoencoder \cite{kingma2013auto} trained on normal data, defining the embedding as the posterior mean $g_{\text{VAE}}(\mathbf{W}) = \mathbb{E}[\mathbf{z} \mid \mathbf{W}]$. This captures nonlinear dependencies while providing a compressed, denoised latent space for retrieval. \textbf{(ii) UMAP Manifold Representation:} Alternatively, we employ UMAP \cite{mcinnes2018umap} to learn an embedding $g_{\text{UMAP}}$ that preserves local neighborhood topology. This manifold-based approach captures geometric similarity between system states without requiring a generative model, as illustrated in the latent projections of \Cref{fig:latent_space_dummy}.

\subsection{Conditional Neighborhood Retrieval}
\label{CNC}

To approximate $p(\mathbf{W}_j \mid \mathbf{W}_t^{(-j)})$, we construct a sensor-specific neighborhood $\mathcal{N}_j(\mathbf{W}_t)$ from a reference set of normal windows $\mathcal{D}_{\text{norm}}$. For an anomalous window $\mathbf{W}_t$, we retrieve the $K$ nearest neighbors conditioned on the context $\mathbf{W}_t^{(-j)}$:

\begin{equation}
\mathcal{N}_j(\mathbf{W}_t)
=
\operatorname{KNN}_K
\left(
\mathbf{W}_t^{(-j)}, \mathcal{D}_{\text{norm}}
\right),
\end{equation}

where $\operatorname{KNN}_K(\cdot)$ returns the $K$ nearest neighbors under the Frobenius distance $\|\cdot\|_F$. This conditioning ensures that retrieved states preserve system dependencies, enabling counterfactual replacements for sensor $j$ that remain within the normal data manifold.

% \subsection{Attribution Computation}

% Given an anomalous window $\mathbf{W}_t$ and its conditional neighborhood $\mathcal{N}(\mathbf{W}_t)$, attribution is computed via dependency-preserving counterfactual replacement. For each sensor $j$, we replace its trajectory $\mathbf{W}_{t,j}$ with trajectories drawn from neighborhood windows while keeping the context $\mathbf{W}_t^{(-j)}$ fixed.

% The empirical conditional attribution is estimated as:

% \begin{equation}
% \hat{\phi}_{t,j}
% =
% \frac{1}{K}
% \sum_{\mathbf{W}' \in \mathcal{N}(\mathbf{W}_t)}
% \left[
% f(\mathbf{W}_t)
% -
% f(\mathbf{W}_t^{(-j)}, \mathbf{W}'_j)
% \right].
% \end{equation}

% To obtain the attribution tensor $\boldsymbol{\Phi}_t \in \mathbb{R}^{w \times d}$, sensor-level contributions are distributed over temporal indices through localized perturbations, enabling joint sensor–time root cause localization.

\subsection{Attribution Computation}
We operationalize the framework by estimating the empirical conditional attribution $\hat{\phi}_{t,j}$ as defined in \Cref{eq:equation5}. This involves averaging the model's response across the retrieved neighborhood $\mathcal{N}(\mathbf{W}_t)$ while keeping the non-target sensor context $\mathbf{W}_t^{(-j)}$ fixed. By grounding the counterfactual replacement in representative normal instances rather than global averages, the computation ensures that the resulting importance scores reflect realistic system deviations. This model-agnostic approach allows the framework to scale efficiently across high-dimensional sensor suites without making restrictive assumptions about the underlying anomaly detector's architecture.

\subsection{Temporal Root Cause Localization}

Beyond identifying contributing sensors, root cause analysis in time-series requires localizing when anomalous behavior emerges. To achieve this, we extend conditional attribution to the temporal dimension by computing contributions over localized segments within the window.

For each sensor $j$ and relative time index $\tau$, we construct time-localized counterfactuals by replacing the value $x_{t+\tau-1,j}$ (or short temporal segments centered at $\tau$) with corresponding values drawn from neighborhood windows while preserving the remaining context. The temporal attribution is then estimated as

\begin{equation}
\Phi_t(\tau,j)
=
\frac{1}{K}
\sum_{\mathbf{W}' \in \mathcal{N}(\mathbf{W}_t)}
\left[
f(\mathbf{W}_t)
-
f(\mathbf{W}_t^{(-(\tau,j))}, \mathbf{W}'_{(\tau,j)})
\right],
\end{equation}

where $\mathbf{W}_t^{(-(\tau,j))}$ denotes the window with the value at $(\tau,j)$ replaced, and $\mathbf{W}'_{(\tau,j)}$ denotes the corresponding value from the neighborhood window.

Aggregating $\Phi_t(\tau,j)$ across $\tau$ recovers sensor-level attribution, while temporal aggregation across $j$ enables anomaly onset localization, yielding a structured sensor–time explanation map.

\subsection{Computational Complexity Analysis}

Let $N$ denote the number of normal windows, $w$ the window length, and $d$ the number of sensors. Input-space neighborhood retrieval requires computing distances in $\mathbb{R}^{w \times d}$, resulting in a complexity of $\mathcal{O}(Nwd)$ per query. 

When retrieval is performed in a learned representation space $g(\cdot) \in \mathbb{R}^k$ with $k \ll wd$, the cost reduces to $\mathcal{O}(Nk)$. Attribution requires $K$ counterfactual evaluations per sensor, leading to a total cost of $\mathcal{O}(dK)$ model evaluations.

Thus, representation-guided retrieval improves scalability while preserving conditional neighborhood fidelity.

% \subsection{Computational Complexity Analysis}

% Let $N$ denote the number of normal windows in $\mathcal{D}_{\text{norm}}$, $w$ the window length, and $d$ the number of sensors. Constructing conditional neighborhoods via input-space retrieval requires computing distances in $\mathbb{R}^{w \times d}$, yielding a complexity of $\mathcal{O}(Nwd)$ per query. 

% When retrieval is performed in a learned representation space $g(\cdot) \in \mathbb{R}^k$ with $k \ll wd$, the complexity reduces to $\mathcal{O}(Nk)$. Since attribution requires $K$ counterfactual evaluations per sensor, the total attribution cost becomes $\mathcal{O}(dK)$ model evaluations, independent of the original data dimensionality. 

% This representation-guided retrieval therefore improves both scalability and computational efficiency while preserving conditional neighborhood fidelity.

\section{Evaluation Metrics}

%We evaluate RCA performance using established retrieval metrics alongside two novel measures designed for explanation reliability and temporal responsiveness. Unlike traditional ranking-based metrics that ignore attribution magnitude and identification latency, our proposed confidence-aware and temporal metrics account for the concentration of attribution mass and the timeliness of root cause detection, providing a more rigorous assessment of operational utility in multivariate time-series.

We evaluate RCA using standard retrieval metrics and introduce two novel measures for explanation reliability and temporal responsiveness. Unlike ranking metrics that ignore attribution strength and detection delay, our metrics capture attribution concentration and detection timeliness, enabling a more rigorous evaluation of operational utility in multivariate time series.

\subsection{Retrieval-Based Metrics}
Following prior work \cite{xu2022anomaly,han2025root}, we evaluate root cause identification using the \textbf{Top-$K$ Recall} ($R@K$) metric. Let $\mathcal{S}_t \subseteq \{1,\dots,d\}$ denote the set of ground-truth root cause sensors for anomaly window $W_t$, and $\widehat{\mathcal{R}}_t^{(K)}$ be the set of top-$K$ sensors ranked by the aggregated attribution scores $\phi_{t,j}$. We assign a uniform weight $w_{t,j} = 1/|\mathcal{S}_t|$ to each causal sensor. The recall for anomaly $t$ is:
\[
\mathrm{R}@K(t) = \sum_{j \in \mathcal{S}_t} w_{t,j}\, \mathbf{1}(j \in \widehat{\mathcal{R}}_t^{(K)})
\]
The dataset-level score $\mathrm{R}@K$ is the average of $\mathrm{R}@K(t)$ over all $N$ anomalous windows. We report scores for $K \in \{3, 5, 10\}$, reflecting realistic settings where operators inspect only the top candidate sensors.

% Confidence-Weighted Top-K Recall.
\subsection{Confidence-Aware Metrics}
%Standard retrieval-based metrics evaluate whether the true root cause sensors appear among the highest-ranked variables but do not account for the magnitude of attribution assigned to them. In practical RCA scenarios, explanations that correctly retrieve the causal sensors while assigning them stronger attribution should be preferred. To capture this aspect, we introduce a confidence-aware metric called the \textit{Confidence-Weighted Root Cause Score (CW-RCS)}, which incorporates attribution strength into the evaluation of root cause identification.

Traditional RCA metrics evaluate whether ground-truth sensors appear in the top rankings but ignore the relative \textit{magnitude} of attribution. To reward explanations that assign dominant attribution mass to the correct causes, we propose the \textbf{Confidence-Weighted Root Cause Score (CW-RCS)}.

\paragraph{Definition.}
Let $\phi_{t,j}$ be the attribution for sensor $j$ at time $t$. We define the normalized confidence weight as $\tilde{\phi}_{t,j} = |\phi_{t,j}| / \sum_{\ell=1}^{d} |\phi_{t,\ell}|$. Let $\mathcal{S}_t$ denote the set of ground-truth root cause sensors and $\widehat{\mathcal{R}}_t^{(K)}$ be the set of top-$K$ ranked sensors. The CW-RCS for an anomaly at time $t$ is:
\begin{equation}
\mathrm{CW\text{-}RCS}(t) = \frac{1}{|\mathcal{S}_t|} \sum_{j \in \mathcal{S}_t \cap \widehat{\mathcal{R}}_t^{(K)}} \tilde{\phi}_{t,j}.
\end{equation}
The dataset-level score is the mean over $N$ anomalous windows:
\begin{equation}
\mathrm{CW\text{-}RCS}@K = \frac{1}{N} \sum_{t=1}^{N} \mathrm{CW\text{-}RCS}(t).
\end{equation}
Unlike rank-based metrics, CW-RCS rewards the model for maximizing the attribution gap between true causes and noise.

\paragraph{Property.}
The proposed metric is bounded by the standard Recall@K:
\begin{equation}
\mathrm{CW\text{-}RCS}(t) \leq \mathrm{Recall}@K(t).
\end{equation}
This holds because for any sensor $j$, the normalized attribution $\tilde{\phi}_{t,j} \leq 1$. Consequently, the sum of weights for correctly retrieved sensors in $\mathcal{S}_t \cap \widehat{\mathcal{R}}_t^{(K)}$ is naturally upper-bounded by the cardinality of that intersection. $\mathrm{CW\text{-}RCS}$ thus penalizes cases where the detector identifies the correct sensor but fails to distinguish it clearly from spurious attributions.

\subsection{Temporal Identification Metrics}
Explanation quality in time-series RCA depends on both the \textit{latency} of identification and its \textit{consistency} throughout the anomaly. We evaluate these via two complementary measures:

\paragraph{Components: Early Identification ($E$) and Persistence ($A$).}
% Let $t_0$ be the anomaly onset, $T_a$ its duration, and $t^*$ the first timestamp where any ground-truth sensor $j \in \mathcal{S}$ appears in the Top-$K$ set $\widehat{\mathcal{R}}_t^{(K)}$. We define the \textbf{Early Identification Score} as $E = \max(0, 1 - \frac{t^* - t_0}{T_a})$, where $E=0$ if the cause is never identified. To measure stability, the \textbf{Temporal Persistence Score} ($A$) represents the fraction of anomalous steps with a correct Top-$K$ attribution:
% \begin{equation}
% A = \frac{1}{T_a} \sum_{t=t_0}^{t_0 + T_a - 1} \mathbf{1}(\mathcal{S} \cap \widehat{\mathcal{R}}_t^{(K)} \neq \emptyset).
% \end{equation}

Let $t_0$ be the anomaly onset, $T_a$ its duration, and $t^*$ the first timestamp where any ground-truth sensor $j \in \mathcal{S}$ appears in the Top-$K$ set $\widehat{\mathcal{R}}_t^{(K)}$. We define these metrics as:
\begin{equation}
E = \max\left(0, 1 - \frac{t^* - t_0}{T_a}\right), \qquad A = \frac{1}{T_a} \sum_{t=t_0}^{t_0 + T_a - 1} \mathbf{1}(\mathcal{S} \cap \widehat{\mathcal{R}}_t^{(K)} \neq \emptyset).
\end{equation}

\paragraph{Harmonic Combined Temporal Score (TemporalHM).}
To ensure responsiveness and stability are treated as mutually necessary, we propose the \textbf{TemporalHM}, generalized via the $F_\beta$ form:
\begin{equation}
\mathrm{TemporalHM}_\beta = \frac{(1+\beta^2)EA}{\beta^2 E + A + \varepsilon}
\end{equation}
where $\varepsilon$ provides numerical stability and $\beta$ allows prioritizing latency ($\beta > 1$) or persistence ($\beta < 1$). This formulation heavily penalizes methods that achieve high early scores but fail to maintain attribution consistency, or vice-versa.

\section{Experiments}

We evaluate the proposed RCA framework to answer the following questions:

% \begin{itemize}
%     \item \textbf{RCA Accuracy:} Can the method correctly identify the true root cause sensors (Top@KR)?
%     \item \textbf{Attribution Confidence:} Does the method concentrate attribution mass on the true causes (CW-RCS@K)?
%     \item \textbf{Temporal Localization:} How early and consistently are root causes detected during the anomaly period (TemporalHM)?
%     \item \textbf{Model Generality:} Does the framework remain effective across different anomaly detection models?
%     %\item \textbf{Efficiency:} Does contextual retrieval reduce attribution computation cost?
% \end{itemize}

\begin{itemize}
    \item \textbf{RQ1 (RCA Accuracy):} To what extent can the framework accurately identify the ground-truth root cause sensors (measured by Top@$K$R)?
    \item \textbf{RQ2 (Attribution Confidence):} Does the method effectively concentrate attribution mass on true causal variables, reducing noise in the explanation (measured by CW-RCS@K)?
    \item \textbf{RQ3 (Temporal Dynamics):} How early and consistently does the framework localize root causes throughout the duration of an anomaly (measured by TemporalHM)?
    \item \textbf{RQ4 (Model Agnosticism):} Is the framework robust and effective when integrated with diverse underlying anomaly detection architectures?
\end{itemize}

The following subsections describe the datasets, models and baselines, and implementation details.

\subsection{Datasets}
We evaluate our framework on two widely-used benchmarks: SWaT (Secure Water Treatment) \cite{xie2020multivariateswat}, containing 51 sensors from a continuous industrial process with 11 days of operation, and MSDS \cite{nedelkoski2020multi} (Multi-Source Distributed System), a high-dimensional dataset capturing 10 metrics across 12 nodes in a distributed computing environment. In our experiments we utilize system metrics modality.

%We evaluate our framework on two widely-used benchmarks: SWaT (Secure Water Treatment) \cite{xie2020multivariateswat}, containing 51 sensors from a continuous 11-day industrial process, and MSDS (Multi-Source Distributed System) \cite{nedelkoski2020multi}, a high-dimensional dataset capturing 10 metrics across 12 nodes. In our experiments, we utilize the system metrics modality.

% We evaluate our method on two widely used multivariate time-series anomaly detection benchmarks: \textbf{SWaT} \cite{xie2020multivariateswat} and \textbf{MSDS}, both representing real-world cyber-physical systems.

% \paragraph{SWaT.}
% The Secure Water Treatment (SWaT) dataset \cite{xie2020multivariateswat} is collected from a scaled industrial water treatment plant consisting of six interconnected process stages. The dataset contains 51 sensors and actuators recorded at 1 Hz. We evaluate our method on the attack phase, which includes 36 labeled cyber-physical attack scenarios with precise start and end timestamps.

% \paragraph{MSDS.}
% The Multi-Source Distributed System (MSDS) dataset \cite{nedelkoski2020multi} contains monitoring data collected from an OpenStack-based distributed system. It provides synchronized multi-source observability data including system metrics, logs, and distributed traces generated under controlled workloads and injected failures. The dataset is designed for operational analytics tasks such as anomaly detection and root cause analysis in complex distributed systems. In our experiments we utilize the system metrics modality.

\subsection{Models and Baselines}

\paragraph{Backbone Models:} To demonstrate the detector-agnostic nature of the proposed RCA framework,we employ a diverse suite of anomaly detection architectures, including Autoencoder (AE) \cite{hinton2006reducing}, Variational Autoencoder (VAE) \cite{kingma2013auto}, LSTM-based detectors \cite{jacob2020exathlon}, TCN \cite{bai2018tcn}, and Transformer-based \cite{xu2022anomaly} models. This selection ensures the attribution framework is tested against varying inductive biases and internal representations.

\paragraph{RCA Baseline:} We compare our approach against representative RCA baselines spanning both causal inference and attribution-based methods: $\epsilon$-Diagnosis \cite{shan2019diagnosis}, RCD \cite{ikram2022root}, CIRCA \cite{li2022causal}, and AERCA \cite{han2025root} for causal root cause localization, as well as KernelSHAP \cite{lundberg2017unifiedshap} and ShaTS \cite{de2025shats} for feature attribution-based explanations. 

\paragraph{Proposed Variants:} We evaluate two framework variants:\textbf{CondAttr-VAE} uses a learned VAE latent space for neighborhood search, while \textbf{CondAttr-UMAP} leverages a UMAP manifold embedding. This design ensures the framework remains model-agnostic and decoupled from the specific anomaly detection backbone. Refer to Appendix A for implementation and hyperparameter details.

\subsection{Implementation Details}

All sensor variables are normalized using standard z-score normalization. Sliding windows of length $w=50$ are constructed to capture temporal system dynamics. Anomaly detection models are trained using only normal operation data.

For contextual retrieval, the neighborhood size is fixed to nearest contextual neighbor ($K=3$). In the VAE variant, the latent representation is learned using a bottleneck dimension of $z=8$, while the UMAP variant constructs a low-dimensional manifold embedding of the input windows for neighbor search.

Experiments are conducted in PyTorch on an NVIDIA RTX A6000 GPU; see Appendix B for hyperparameter studies.

\section{Results and Analysis}

\subsection{Root Cause Identification Performance}
\begin{table}[!ht]
\centering
\small
\setlength{\tabcolsep}{1.5pt}
\renewcommand{\arraystretch}{1}
\caption{\textbf{Root Cause Identification Performance}. Top-$K$ Recall reported per dataset. \textbf{Best} and \underline{second-best} results are highlighted.}
\begin{tabular}{lcccccc}
\toprule
\textbf{Method}
& \multicolumn{3}{c}{\textbf{SWaT}}
& \multicolumn{3}{c}{\textbf{MSDS}} \\
\cmidrule(lr){2-4}\cmidrule(lr){5-7}
& \textbf{Top@3R} & \textbf{Top@5R} & \textbf{Top@10R}
& \textbf{Top@3R} & \textbf{Top@5R} & \textbf{Top@10R} \\
\midrule
$\epsilon$-Diagnosis & 0.125 & 0.125 & 0.375 & 0.266 & 0.452 & 1.000 \\
RCD                 & 0.000 & 0.000 & 0.300 & 0.573 & 0.984 & 1.000 \\
CIRCA               & 0.000 & 0.000 & 0.300 & 0.860 & 0.917 & 1.000 \\
AERCA               & 0.290 & 0.330 & 0.455 & 0.908 & 0.974 & 1.000 \\
KernelSHAP          & 0.055 & 0.064 & 0.138 & 0.311 & 0.467 & 1.000 \\
ShaTS               & 0.393 & 0.513 & 0.601 & 0.915 & 0.986 & 1.000 \\
\rowcolor{Aquamarine!10} CondAttr-VAE               & \textbf{0.537} & \textbf{0.601} & \textbf{0.694} & \underline{0.948} & \textbf{1.000} & 1.000 \\
\rowcolor{Aquamarine!10} CondAttr-UMAP              & \underline{0.481} & \underline{0.523} & \underline{0.638} & \textbf{0.956} & \textbf{1.000} & 1.000 \\
\bottomrule
\end{tabular}
\label{tab:rci_performance_transposed}
\end{table}

As shown in \Cref{tab:rci_performance_transposed}, \textbf{CondAttr-UMAP} and \textbf{CondAttr-VAE} consistently outperform all baselines across both benchmarks. On SWaT, our manifold-guided approaches provide a significant lift over the strongest baseline (ShaTS), with improvements of up to \textbf{36.64\%} in Top-3 Recall. Similar gains on the high-dimensional MSDS dataset underscore the scalability of our contextual retrieval strategy. These results empirically validate that conditioning attributions on the learned system manifold, rather than using marginal perturbations, better preserves inter-sensor dependencies and reduces out-of-distribution artifacts, leading to more precise root-cause localization.

Additional comparisons between conditional and unconditional retrieval strategies are provided in Appendix~B.6, further illustrating the benefits of conditioning the reference set on the anomalous context.
%\Cref{tab:rci_performance_transposed} shows that the proposed methods \textbf{CondAttr-VAE} and \textbf{CondAttr-UMAP} consistently outperform existing RCA approaches. On SWaT, CondAttr-UMAP improves Top@3 recall by approximately \textbf{22.39\%} over the strongest attribution baseline (ShaTS), while CondAttr-VAE improves Top@10 recall by \textbf{12.31\%}. Similar improvements of \textbf{X\%} are observed on MSDS across Top-$K$ metrics. These results indicate that contextual conditional attribution better preserves system dependencies compared to marginal perturbation-based explanations, leading to more accurate root cause identification.

\subsection{Confidence-Aware Evaluation}
\Cref{tab:confidence_aware_evaluation} evaluates attribution quality via the proposed \textit{CW-RCS@K} metric, denoted as \textit{CW@K} in the table for brevity. \textbf{CondAttr-UMAP} consistently outperforms the strongest baseline,  \cite{de2025shats}, nearly doubling the scores across all $K$ on SWaT \cite{xie2020multivariateswat}. While absolute scores are lower than standard Recall@$K$ due to the strictness of confidence weighting, the substantial performance gap confirms that our conditional framework effectively concentrates attribution mass on ground-truth root causes. Unlike marginal baselines that disperse importance across spurious correlations, our manifold-guided approach enhances the \textbf{separability} of root causes from background noise, providing more decisive and operationally actionable explanations.

% Metric comparisons.
\begin{table}[!ht]
\centering
\small
\setlength{\tabcolsep}{1pt}
\renewcommand{\arraystretch}{1}
\caption{\textbf{Confidence-Aware Evaluation}. CW-RCS@$K$ per dataset. \textbf{Best} and \underline{second-best} results are highlighted.}
\begin{tabular}{lcccccc}
\toprule
\textbf{Method}
& \multicolumn{3}{c}{\textbf{SWaT}}
& \multicolumn{3}{c}{\textbf{MSDS}} \\
\cmidrule(lr){2-4}\cmidrule(lr){5-7}
& \textbf{  CW@3} & \textbf{CW@5} & \textbf{CW@10  }
& \textbf{  CW@3} & \textbf{CW@5} & \textbf{CW@10} \\
\midrule
KernelSHAP & 0.004 & 0.004 & 0.009 & 0.050 & 0.093 & 0.144 \\
ShaTS      & 0.122 & 0.134 & 0.141 & 0.462 & 0.489 & 0.517 \\
\rowcolor{Aquamarine!10} CondAttr-VAE      & \textbf{0.245} & \textbf{0.253} & \underline{0.257} & \underline{0.551} & \underline{0.607} & \underline{0.664} \\
\rowcolor{Aquamarine!10} CondAttr-UMAP     & \underline{0.243} & \underline{0.250} & \textbf{0.258} & \textbf{0.569} & \textbf{0.624} & \textbf{0.668} \\
\bottomrule 
\end{tabular}
\label{tab:confidence_aware_evaluation}
\end{table}

\subsection{Cross-Model Explanation Consistency}
Table \ref{tab:cross_model_consistency} evaluates the robustness of our framework across diverse anomaly detection backbones.
%, including stochastic (VAE/AE), recurrent (LSTM), and attention-based (Transformer) models. 
While marginal attribution methods exhibit high variance and poor confidence scores, \textbf{CondAttr-UMAP} maintains superior performance across all detectors. Notably, on high-capacity models like the Transformer and TCN, our method achieves nearly double the CW@3 scores compared to ShaTS \cite{de2025shats}. This consistency suggests that by decoupling the explanation logic from the detector’s internal parameters and grounding it in the system's normal manifold, we mitigate model-specific biases. Our approach thus provides a model-agnostic layer of reliability, ensuring that root cause attributions remain stable regardless of the underlying detection architecture.
\begin{table}[t]
\centering
\small
\setlength{\tabcolsep}{3pt}
\renewcommand{\arraystretch}{1.05}
\caption{\textbf{Cross-model Consistency}. Top-3 Recall and confidence-weighted recall (CW-RCS@3) across anomaly detection backbones. \textbf{Best} and \underline{second-best} results are highlighted.}
\resizebox{\linewidth}{!}{%
\begin{tabular}{lcccccccccc}
\toprule
\textbf{Method}
& \multicolumn{2}{c}{\textbf{VAE}}
& \multicolumn{2}{c}{\textbf{AE}}
& \multicolumn{2}{c}{\textbf{LSTM}}
& \multicolumn{2}{c}{\textbf{TCN}}
& \multicolumn{2}{c}{\textbf{Transformer}} \\
\cmidrule(lr){2-3}
\cmidrule(lr){4-5}
\cmidrule(lr){6-7}
\cmidrule(lr){8-9}
\cmidrule(lr){10-11}
& \textbf{Top@3R} & \textbf{CW@3}
& \textbf{Top@3R} & \textbf{CW@3}
& \textbf{Top@3R} & \textbf{CW@3}
& \textbf{Top@3R} & \textbf{CW@3}
& \textbf{Top@3R} & \textbf{CW@3} \\
\midrule
KernelSHAP & 0.055 & 0.004 & 0.097 & 0.038 & 0.041 & 0.006 & 0.046 & 0.017 & 0.092 & 0.011 \\
ShaTS      & 0.393 & 0.122 & 0.332 & 0.158 & 0.407 & 0.126 & 0.462 & 0.138 & 0.475 & 0.196 \\
\rowcolor{Aquamarine!10} CondAttr-VAE      & \textbf{0.537} & \textbf{0.245} & \textbf{0.425} & \underline{0.216} & \textbf{0.490} & \textbf{0.229} & \underline{0.555} & \underline{0.284} & \textbf{0.564} & \textbf{0.325} \\
\rowcolor{Aquamarine!10} CondAttr-UMAP   & \underline{0.481} & \underline{0.238} & \underline{0.402} & \textbf{0.222}  & \underline{0.462} & \underline{0.213} & \textbf{0.569} & \textbf{0.342} & \underline{0.518} & \underline{0.303} \\
\bottomrule
\end{tabular}%
}
\label{tab:cross_model_consistency}
\end{table}

\subsection{Temporal Localization Analysis}
As shown in \Cref{tab:temporal_localization}, our manifold-guided strategies lead in \textbf{TemporalHM} (denoted \textit{TempHM} in tables), with the UMAP variant achieving the highest fidelity. The substantial margin over ShaTS \cite{de2025shats}, particularly at $K=3$, underscores our framework's capacity to minimize localization latency while ensuring persistence. Unlike marginal perturbations that often yield transient or "flickering" attributions, our conditional approach stabilizes explanations by grounding them in the system's learned normal manifold. This ensures root causes are identified promptly at onset and tracked reliably as the anomaly evolves, providing the stability required for real-time industrial monitoring.
%\vspace{-1em}
\begin{table}[!ht]
\centering
\small
\setlength{\tabcolsep}{1pt}
\renewcommand{\arraystretch}{1}
\caption{\textbf{Temporal Localization Analysis}. \textbf{Best} and \underline{second-best} results are highlighted.}
\begin{tabular}{llccc}
\toprule
\textbf{Dataset} & \textbf{Method} & \textbf{TempHM@3} & \textbf{TempHM@5} & \textbf{TempHM@10} \\
\midrule
\multirow{4}{*}{\textbf{SWaT}}
& KernelSHAP & 0.064 & 0.078 & 0.171 \\
& ShaTS      & 0.422 & 0.519 & 0.562 \\
& \cellcolor{Aquamarine!10} CondAttr-VAE      & \cellcolor{Aquamarine!10}\underline{0.503} & \cellcolor{Aquamarine!10}\underline{0.562} & \cellcolor{Aquamarine!10}\textbf{0.650} \\
& \cellcolor{Aquamarine!10} CondAttr-UMAP     & \cellcolor{Aquamarine!10}\textbf{0.504} & \cellcolor{Aquamarine!10}\textbf{0.568} & \cellcolor{Aquamarine!10}\underline{0.629} \\
\bottomrule
\end{tabular}
\label{tab:temporal_localization}
\end{table}
%\vspace{-1em}

%\subsection{Representation Impact Study}
% VAE vs UMAP comparison.

\section{Industrial Case Study: Blast Furnace Monitoring}

To assess the practical utility of the proposed framework, we conduct an industrial case study on real-world blast furnace monitoring data. The study evaluates root cause attribution performance in a complex thermo-chemical manufacturing environment under domain-expert supervision.

% \subsection{System Overview and Data Description}

% We evaluate the proposed framework in collaboration with Paul Wurth using data from an operational blast furnace used in large-scale steel manufacturing. The furnace represents a complex thermo-chemical system instrumented with over 100 heterogeneous sensors deployed across multiple structural levels, monitoring critical variables such as temperature, gas composition, pressure, and reaction dynamics throughout the production cycle.

% The dataset comprises approximately 2 million time-stamped observations spanning multiple operating cycles. Due to industrial confidentiality constraints, the data cannot be publicly released. Data curation was performed in close collaboration with domain experts, including sensor validation, noise filtering, channel synchronization, and process-level feature interpretation, ensuring physically consistent inputs for downstream anomaly detection and root cause analysis.

\subsection{System Overview and Data Description}
We validate our framework in collaboration with \textbf{Paul Wurth} using data from an operational blast furnace. This complex thermo-chemical system is monitored by over 100 heterogeneous sensors (temperature, gas composition, pressure) across multiple structural levels. The dataset contains 2 million observations over several production cycles. While industrial confidentiality precludes public release, the data underwent rigorous curation, including sensor validation, noise filtering, and synchronization, in close coordination with domain experts to ensure physically consistent inputs for root cause analysis.

% \subsection{Data Preparation and Modeling Pipeline}

% Following curation, standard preprocessing steps were applied, including normalization, missing-value handling, and temporal alignment across sensor channels. Sliding windows were constructed to capture process dynamics, with window length selected in consultation with domain experts based on typical furnace reaction and material transit times.

% We trained reconstruction-based anomaly detectors, including autoencoder (AE) and variational autoencoder (VAE) models, on normal operating data. Both models achieved stable reconstruction behavior and reliable anomaly sensitivity, providing the detection backbone required for downstream root cause attribution.

\subsection{Modeling Pipeline}
Following standard preprocessing (normalization and alignment), we constructed sliding windows with lengths selected in consultation with domain experts to match furnace reaction and material transit times. We trained reconstruction-based detectors, specifically Autoencoder (AE) and VAE models, on normal operating data. Both architectures achieved stable reconstruction and the necessary sensitivity for downstream RCA, serving as the model-agnostic backbones for our attribution framework.

% \subsection{Operational Need for Root Cause Analysis}

% While anomaly detection can flag abnormal furnace behavior, operational intervention requires identifying the causal sensors driving the deviation. Blast furnace production cycles span several hours, and late-stage anomaly discovery can result in material loss, energy waste, and unplanned downtime. We therefore applied the proposed root cause attribution framework to detected anomalies to identify sensors deviating from their conditional operating patterns, enabling earlier diagnosis and more targeted corrective action within the manufacturing process.

\subsection{Operational Requirements at Paul Wurth}
For Paul Wurth, anomaly detection is only actionable if it identifies the specific causal sensors driving a deviation. Given that blast furnace thermo-chemical cycles span several hours, late-stage discovery leads to irreversible material loss and significant energy waste. Our framework addresses this by isolating sensors deviating from their \textit{context-dependent} patterns. This provides operators with precise localization needed for targeted intervention during the critical early phases of a process deviation, transforming a binary alert into a diagnostic insight.

% \subsection{RCA Validation and Stress Testing}

% We evaluated the proposed RCA framework on historical furnace anomalies and verified the identified root cause sensors against expert annotations and process logs. The method consistently localized sensors exhibiting abnormal deviations from their normal operational context. 

% To further validate attribution reliability, we conducted controlled stress tests by injecting synthetic anomalies into selected sensors and correlated sensor groups within normal operating data. In these scenarios, the framework accurately recovered the perturbed sensors and demonstrated the ability to detect anomalous deviations at early stages of manifestation. Representative sensor–time attribution heatmaps for both real and injected anomalies are presented in Figure~X, illustrating the interpretability and diagnostic utility of the proposed approach.

\subsection{Validation and Stress Testing}
We validated our framework against expert-annotated historical anomalies and process logs. To further quantify reliability, we conducted controlled stress tests by injecting synthetic anomalies into individual sensors and correlated groups. Across both scenarios, the framework consistently recovered the ground-truth drivers, demonstrating high sensitivity to deviations at their earliest stages. As illustrated by the attribution heatmaps in \Cref{fig:pw_graphs}, our approach provides the interpretability required for industrial diagnosis, accurately isolating causal sensors from normal system noise. Appendix C provides additional anomaly cases and experimental results.

% \begin{figure}[t]
% \centering

% \begin{subfigure}[t]{0.48\linewidth}
%     \centering
%     \includegraphics[width=\linewidth, height=3cm]{images/art_anom_01.png}
%     \caption{Real Anomaly Dataset}
%     \label{fig:left_image}
% \end{subfigure}
% \hfill
% \begin{subfigure}[t]{0.48\linewidth}
%     \centering
%     \includegraphics[width=\linewidth, height=3cm]{images/art_anom_03.png}
%     \caption{synthetic anomalies}
%     \label{fig:right_image}
% \end{subfigure}
% \label{fig:pw_graphs}
% \end{figure}

\begin{figure}[t]
    \centering
    \includegraphics[width=\linewidth]{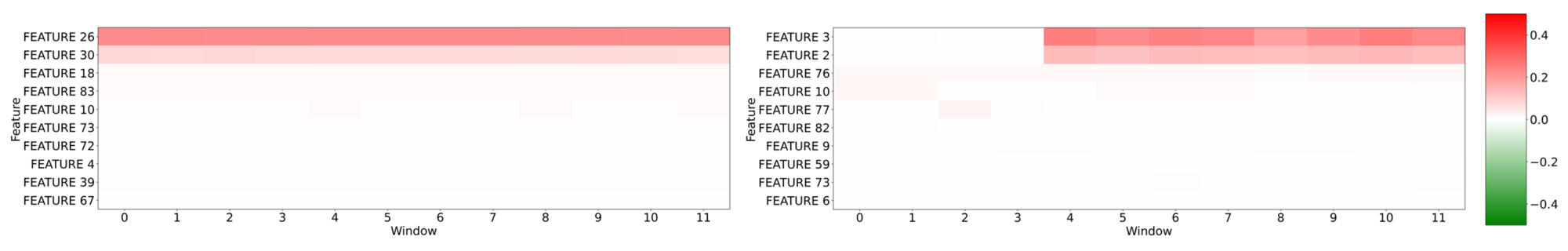}
    \caption{\textbf{Root cause feature heatmaps for representative Paul Wurth samples.} The left panel shows a real test sample with \texttt{feature26} identified as the root cause, while the right panel shows a synthetic anomaly with perturbations injected in \texttt{feature2} and \texttt{feature3} after a given time point.}
    \label{fig:pw_graphs}
\end{figure}

\subsection{Industrial Impact and Portability}
The Paul Wurth case study validates the practical utility of conditional attribution in mitigating downtime and resource loss through early diagnostic intervention. Beyond blast furnace monitoring, the framework’s detector-agnostic design ensures seamless integration with diverse anomaly detection architectures. This versatility supports broad transferability to other complex multivariate industrial systems, providing a scalable solution for high-fidelity root cause diagnosis where system dependencies are critical.

% \subsection{Industrial Impact and Transferability}

% Due to industrial confidentiality constraints, the blast furnace dataset cannot be publicly released. Nevertheless, the case study highlights the practical value of conditional attribution for diagnosing complex process anomalies and enabling earlier intervention to mitigate downtime and resource loss. As demonstrated in our experiments, the proposed framework is detector-agnostic, allowing seamless integration with diverse anomaly detection models, thereby supporting transferability across industrial monitoring and multivariate RCA scenarios.

\section{Conclusion}
In this paper, we addressed the fundamental limitations of marginal attribution in time-series RCA, which often lead to unreliable, out-of-distribution explanations. We proposed a conditional attribution framework grounded in the system's learned manifold, utilizing VAE and UMAP representations to retrieve contextually relevant normal states for counterfactual construction. To rigorously assess the utility of these explanations, we introduced two novel metrics: the confidence-weighted \textit{CW-RCS} and the stability-focused \textit{TemporalHM}. Our experiments on the SWaT and MSDS benchmarks, along with a case study on industrial blast furnace data from \textbf{Paul Wurth}, demonstrate that our approach consistently improves root-cause identification accuracy and temporal responsiveness. Qualitative evaluation by domain experts further confirms the practical utility of our approach, as the identified root causes showed strong alignment with known physical process failures. By providing high-fidelity, dependency-preserving explanations, our framework bridges the gap between complex anomaly detection models and actionable industrial diagnostics. Future work will explore the extension of this manifold-guided strategy to online, incremental learning settings where system dynamics evolve over time.

\section{Acknowledgements}
This work was partially funded by the German Federal Ministry of Research, Technology and Space (BMFTR) under Grant Agreement No. 16IW24009 (COPPER).

\bibliographystyle{splncs04}
\bibliography{mybibliography}
%% Note that this preceding line implies that you store your BibTeX references in a file called 'mybibliography.bib'. If you instead store your references in a file with a different name, for instance 'references.bib', the preceding line should read '\bibliography{references}'. Whatever you do, DO NOT put the file name extension .bib inside the \bibliography command; this will trip up LaTeX compilers. 
%
% If you do not want to use BibTeX, you can also type up the bibliography exactly as you see fit, using the following structure:

%----Supplementary Section-------
% \documentclass[runningheads]{llncs}

% \usepackage{graphicx}
% \usepackage{amsmath}
% \usepackage{amssymb}
% \usepackage{booktabs}
% \usepackage{multirow}

% \usepackage{graphicx}
% \usepackage{subcaption}

\clearpage
\appendix

%\begin{document}
\title{Appendix:
Conditional Attribution for Root Cause Analysis in Time-Series Anomaly Detection}
\titlerunning{Conditional Attribution for Time-Series RCA}
\author{}
\institute{}

\maketitle

\section{Additional Model Details}

\subsection{Model Details}

\subsubsection{VAE Architecture and Hyperparameters}

% The variational autoencoder (VAE) is used as one of the backbone anomaly detection models in our experiments. It consists of an encoder that maps an input time-series window $\mathbf{X} \in \mathbb{R}^{T \times D}$ to the parameters of a latent Gaussian distribution, namely the mean vector $\boldsymbol{\mu}$ and log-variance vector $\log \boldsymbol{\sigma}^2$. A latent sample is then drawn using the reparameterization trick and passed through the decoder to reconstruct the input window, yielding $\hat{\mathbf{X}}$.

The variational autoencoder (VAE) employed in \textit{CondAttr-VAE} acts as the underlying latent generative model for conditional attribution. Given an input time-series window $\mathbf{X} \in \mathbb{R}^{T \times D}$, the encoder maps the input to the parameters of a latent Gaussian distribution, namely the mean vector $\boldsymbol{\mu}$ and log-variance vector $\log \boldsymbol{\sigma}^2$. A latent representation is then sampled via the reparameterization trick and passed through the decoder to obtain the reconstruction $\hat{\mathbf{X}}$.

The VAE is trained using a combined objective consisting of a reconstruction loss, a KL regularization term, and a time-axial consistency loss:
% \begin{equation}
% \mathcal{L}_{\mathrm{total}} =
% \lambda_{\mathrm{rec}} \mathcal{L}_{\mathrm{rec}}
% + \mathcal{L}_{\mathrm{KL}}
% + \mathcal{L}_{\mathrm{time}},
% \end{equation}
\begin{equation}
\mathcal{L}_{\mathrm{total}} =
\lambda_{\mathrm{rec}} \mathcal{L}_{\mathrm{rec}}
+ \lambda_{\mathrm{KL}} \mathcal{L}_{\mathrm{KL}}
+ \lambda_{\mathrm{time}} \mathcal{L}_{\mathrm{time}},
\end{equation}
where the reconstruction loss is defined as
\begin{equation}
\mathcal{L}_{\mathrm{rec}} = \sum_{t=1}^{T}\sum_{d=1}^{D} \left(X_{t,d} - \hat{X}_{t,d}\right)^2,
\end{equation}
and the KL divergence term is
\begin{equation}
\mathcal{L}_{\mathrm{KL}} =
-\frac{1}{2}\sum_{j=1}^{L}
\left(
1 + \log \sigma_j^2 - \mu_j^2 - \exp(\log \sigma_j^2)
\right),
\end{equation}
where $L$ denotes the latent dimension. To further encourage temporal consistency in the reconstruction, we include a time-axial loss defined as
\begin{equation}
\mathcal{L}_{\mathrm{time}} =
\sum_{t=1}^{T}
\left(
\frac{1}{D}\sum_{d=1}^{D} X_{t,d}
-
\frac{1}{D}\sum_{d=1}^{D} \hat{X}_{t,d}
\right)^2.
\end{equation}

Table~\ref{tab:vae_hyperparameters} summarizes the architectural and training hyperparameters of the VAE used in our experiments.

\begin{table}[!ht]
\centering
\small
\setlength{\tabcolsep}{1pt}
\renewcommand{\arraystretch}{1}
\caption{\textbf{VAE architecture and training hyperparameters.} Hyperparameter configuration of the timeVAE model used in our experiments.}
\begin{tabular}{ll}
\toprule
\textbf{Hyperparameter} & \textbf{Value} \\
\midrule
Input dimension            & 50 \\
Latent dimension             & 8 \\
Hidden layer sizes           & [50, 100, 200] \\
Reconstruction weight        & 3.0 \\
KL-Divergence weight        & 1.0 \\
Time-Axial weight        & 1.0 \\
Batch size                   & 512 \\
Maximum epochs               & 300 \\
\bottomrule
\end{tabular}
\label{tab:vae_hyperparameters}
\end{table}

\subsubsection{UMAP Architecture and Hyperparameters}

The UMAP-based variant, referred to as \textit{CondAttr-UMAP}, employs Uniform Manifold Approximation and Projection (UMAP) as the nonlinear dimensionality reduction module underlying conditional attribution. Given a multivariate time-series window of length $T$ with $D$ features, each input sample is reshaped into a flattened vector of dimension $T \times D$ before being projected into a lower-dimensional embedding space. In our implementation, UMAP is configured with \texttt{n\_neighbors = 30}, \texttt{n\_components = 10}, \texttt{min\_dist = 0.1}, and \texttt{random\_state = 42}. All remaining parameters are retained at their default settings, including the Euclidean distance metric, spectral initialization, learning rate of $1.0$, spread of $1.0$, and automatic selection of the number of optimization epochs.

\begin{table}[!ht]
\centering
\small
\setlength{\tabcolsep}{1pt}
\renewcommand{\arraystretch}{1}
\caption{\textbf{UMAP hyperparameters.} Hyperparameter configuration of the UMAP model used in our experiments.}
\begin{tabular}{ll}
\toprule
\textbf{Hyperparameter} & \textbf{Value} \\
\midrule
Output dimension       & 10 \\
Number of neighbors    & 30 \\
Minimum distance       & 0.1 \\
Distance metric        & Euclidean \\
Initialization         & Spectral \\
Learning rate          & 1.0 \\
Spread                 & 1.0 \\
Optimization epochs    & Automatic \\
Random state           & 42 \\
\bottomrule
\end{tabular}
\label{tab:umap_hyperparameters}
\end{table}

\subsection{Anomaly Detection Scores}

Table~\ref{tab:ad_performance} reports the anomaly detection performance of different backbone models on the SWaT dataset.

\begin{table}[!ht]
\centering
\small
\setlength{\tabcolsep}{1pt}
\renewcommand{\arraystretch}{1}
\caption{\textbf{Anomaly detection performance on the SWaT dataset} across different backbone models.}
\begin{tabular}{lcccc}
\toprule
\textbf{Model} & \textbf{Precision} & \textbf{Recall} & \textbf{F1-Score} & \textbf{ROC-AUC} \\
\midrule
VAE         & 0.952 & 0.901 & 0.926 & 0.954 \\
AE          & 0.960 & 0.910 & 0.934 & 0.962 \\
LSTM        & 0.979 & 0.923 & 0.948 & 0.980 \\
TCN         & 0.989 & 0.961 & 0.975 & 0.990 \\
Transformer & 0.969 & 0.949 & 0.959 & 0.983 \\
\bottomrule
\end{tabular}
\label{tab:ad_performance}
\end{table}

\section{Ablation}

% \subsection{Inference Time}

% Table~\ref{tab:inference_time} reports the inference time comparison across different explanation methods.

% \begin{table}[!ht]
% \centering
% \small
% \setlength{\tabcolsep}{1pt}
% \renewcommand{\arraystretch}{1}
% \begin{tabular}{lccc}
% \toprule
% \textbf{Method} & \textbf{SWAT} & \textbf{MSDS} & \textbf{Industrial Dataset} \\
% \midrule
% KernelSHAP & fill & fill & fill \\
% ShaTS      & fill & fill & fill \\
% MyVAE      & fill & fill & fill \\
% MyUMAP     & fill & fill & fill \\
% \bottomrule
% \end{tabular}
% \caption{\textbf{Inference time comparison across explanation methods.} Values denote the average inference time per anomalous window.}
% \label{tab:inference_time}
% \end{table}

\subsection{Additional Dataset Details}

We evaluate our method on two real-world multivariate time-series benchmarks:

\textbf{SWaT} (Secure Water Treatment) is derived from a scaled-down but high-fidelity six-stage water treatment testbed. The SWaT.A1\_Dec 2015 release consists of 11 days of continuous operation, of which 7 days correspond to normal operation and 4 days contain attack scenarios. The dataset provides labeled measurements from 51 sensors and actuators, and includes 41 attacks during the abnormal period.

\textbf{MSDS} (Multi-Source Distributed System) is collected from a complex distributed OpenStack environment and is designed to facilitate AIOps tasks, including anomaly detection and root cause analysis. It comprises multi-source observability data, including metrics, logs, and distributed traces, along with workload and fault scripts that provide ground truth. In our experiments, we use the system metrics modality and adopt the benchmark setting with 10 variables.

\begin{table}[!ht]
\centering
\small
\setlength{\tabcolsep}{8pt}
\renewcommand{\arraystretch}{1}
\caption{\textbf{Summary of dataset statistics.} Number of features corresponds to the number of columns in each multivariate time-series dataset.}
\begin{tabular}{lcc}
\toprule
\textbf{Dataset} & \textbf{Features} & \textbf{Timesteps} \\
\midrule
SWaT & 51 & 49,500 \\
MSDS & 10 & 29,268 \\
\bottomrule
\end{tabular}
\label{tab:dataset_stats}
\end{table}

\subsection{Inference Time}

Table~\ref{tab:inference_time} presents a comparison of explanation methods in terms of both computational efficiency and localization performance. Specifically, for each dataset, we report the average inference time (in seconds) required to explain a single anomalous window, along with the corresponding Top-$K$ Recall. While inference time reflects the practical scalability of an explanation method, Top-$K$ Recall indicates its ability to recover the true anomalous variables among the highest-ranked explanations.

\begin{table}[!ht]
\centering
\small
\setlength{\tabcolsep}{6pt}
\renewcommand{\arraystretch}{1}
\caption{\textbf{Comparison of inference latency and Top@3R across explanation methods.} We report the mean inference time (seconds per window) and Top@3R accuracy for both benchmarks.}
\begin{tabular}{lcc|cc}
\toprule
\multirow{2}{*}{\textbf{Method}} 
& \multicolumn{2}{c|}{\textbf{SWAT}} 
& \multicolumn{2}{c}{\textbf{MSDS}} \\
\cmidrule(lr){2-3}\cmidrule(lr){4-5}
& \textbf{Time (s)} & \textbf{Top@3R}
& \textbf{Time (s)} & \textbf{Top@3R} \\
\midrule
KernelSHAP        & 18.653 & 0.055 & 0.402 & 0.311 \\
ShaTS             & 21.822 & 0.393 & 0.441 & 0.915 \\
CondAttr-VAE      & 21.952 & 0.537 & 0.461 & 0.948 \\
CondAttr-UMAP     & 22.011 & 0.481 & 0.475 & 0.956 \\
\bottomrule
\end{tabular}
\label{tab:inference_time}
\end{table}

\subsection{Window Size Impact}

Table~\ref{tab:window_size_impact} reports the sensitivity of the proposed conditional attribution methods to the choice of input window size. Specifically, we evaluate CondAttr-VAE and CondAttr-UMAP using Top@3R, CW-RCS@3, and TemporalHM@3 over different temporal window lengths. The results indicate that explanation performance is generally higher for smaller or intermediate window sizes, whereas larger windows tend to degrade root cause localization quality. This behavior suggests a trade-off between capturing sufficient temporal context and preserving attribution specificity. Based on this trade-off, we choose a window size of 50 for reporting the main results in the paper, as it represents a moderate setting that balances temporal context and localization performance. The remaining ablation results for shorter and longer window sizes are provided here for completeness.

\begin{table}[!ht]
\centering
\small
\setlength{\tabcolsep}{1pt}
\renewcommand{\arraystretch}{1}
\caption{\textbf{Sensitivity analysis of RCA performance across varying window sizes.}}
\begin{tabular}{c|ccc|ccc}
\toprule
\multirow{2}{*}{\textbf{Window Size}} 
& \multicolumn{3}{c|}{\textbf{CondAttr-VAE}} 
& \multicolumn{3}{c}{\textbf{CondAttr-UMAP}} \\
\cline{2-7}
& Top@3R & CW@3 & TempHM@3 & Top@3R & CW@3 & TempHM@3 \\
\midrule
5   & 0.592 & 0.318 & 0.436 & 0.523 & 0.312 & 0.413 \\
10  & 0.574 & 0.320 & 0.446 & 0.564 & 0.327 & 0.481 \\
20  & 0.509 & 0.275 & 0.437 & 0.513 & 0.318 & 0.465 \\
50  & 0.537 & 0.245 & 0.503 & 0.481 & 0.243 & 0.504 \\
100 & 0.435 & 0.196 & 0.397 & 0.453 & 0.236 & 0.479 \\
200 & 0.287 & 0.193 & 0.317 & 0.393 & 0.189 & 0.412 \\
500 & 0.287 & 0.177 & 0.279 & 0.425 & 0.239 & 0.410 \\
\bottomrule
\end{tabular}
\label{tab:window_size_impact}
\end{table}

\subsection{Attribution Sizes Impact}

Table~\ref{tab:attribution_size_impact} presents the sensitivity of CondAttr-VAE and CondAttr-UMAP to the attribution sample size. We report Top@3R, CW-RCS@3, and TemporalHM@3 for varying numbers of attribution samples. The results show that increasing the attribution size from very small values leads to initial improvements, particularly for CondAttr-VAE. Beyond this range, however, the gains become marginal, indicating a clear performance plateau for larger attribution sizes. This suggests that moderate attribution sizes are sufficient for obtaining reliable explanations, while further increases offer limited additional benefit.

\begin{table}[!ht]
\centering
\small
\setlength{\tabcolsep}{1pt}
\renewcommand{\arraystretch}{1}
\caption{\textbf{Impact of attribution size on RCA performance.} Results for CondAttr-VAE and CondAttr-UMAP across different attribution sizes, evaluated using Top@3R, CW@3, and TempHM@3. Performance generally stabilizes for larger attribution sizes.}
\begin{tabular}{c|ccc|ccc}
\toprule
\textbf{Attribution Size}
& \multicolumn{3}{c|}{\textbf{CondAttr-VAE}}
& \multicolumn{3}{c}{\textbf{CondAttr-UMAP}} \\
\cline{2-7}
& Top@3R & CW@3 & TempHM@3 & Top@3R & CW@3 & TempHM@3 \\
\midrule
1   & 0.472 & 0.235 & 0.446 & 0.468 & 0.237 & 0.484 \\
2   & 0.513 & 0.246 & 0.489 & 0.476 & 0.241 & 0.496 \\
3   & 0.537 & 0.245 & 0.503 & 0.481 & 0.243 & 0.504 \\
4   & 0.537 & 0.241 & 0.504 & 0.482 & 0.244 & 0.504 \\
5   & 0.537 & 0.239 & 0.499 & 0.493 & 0.244 & 0.501 \\
10  & 0.550 & 0.243 & 0.467 & 0.508 & 0.246 & 0.486 \\
20  & 0.550 & 0.238 & 0.470 & 0.509 & 0.246 & 0.478 \\
50  & 0.550 & 0.241 & 0.461 & 0.511 & 0.245 & 0.468 \\
\bottomrule
\end{tabular}
\label{tab:attribution_size_impact}
\end{table}

\subsection{Input vs Representation-Space Retrieval}

In high-dimensional sensor data, nearest-neighbor retrieval in the \textbf{raw input space} is unreliable because distance metrics (e.g., Euclidean) treat all features uniformly, causing irrelevant sensor fluctuations to distort similarity relationships and produce noisy neighborhoods.

To address this, we perform retrieval in \textbf{learned representation spaces}. A Variational Autoencoder (VAE) maps the input $x \in \mathbb{R}^d$ to a lower-dimensional latent representation $z \in \mathbb{R}^k$, capturing dominant system factors while suppressing noise. We further apply UMAP to preserve the \textbf{local manifold structure} of the latent space.

Retrieval quality is evaluated using \textit{CW-RCS@3}. As shown in Fig.~\ref{fig:retrieval_space_comparison}, representation space retrieval (VAE latent and UMAP manifold) significantly outperforms raw-input retrieval, indicating that learned representations better capture the \textit{intrinsic system manifold} and yield more consistent nearest neighbors for anomaly analysis.

\begin{figure}[!ht]
\centering
\includegraphics[width=0.7\linewidth]{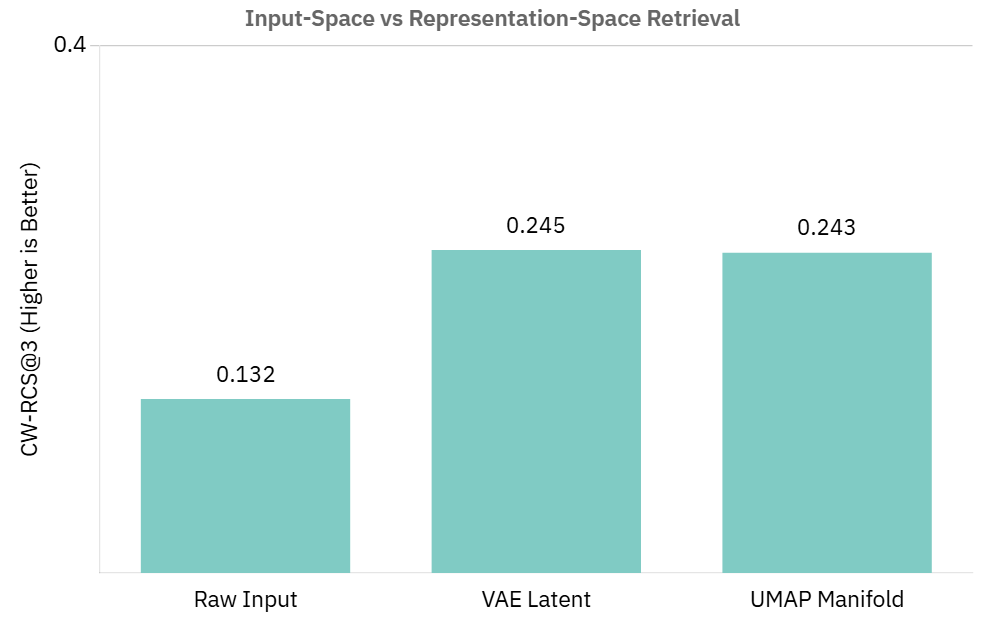}
\caption{Comparison of retrieval performance across spaces. Representation-space retrieval (VAE latent and UMAP manifold) achieves higher CW-RCS than raw input space.}
\label{fig:retrieval_space_comparison}
\end{figure}

\subsection{Unconditional vs Conditional Retrieval}

We study the impact of retrieval context on anomaly attribution.  
\textbf{Unconditional retrieval} constructs the reference set using $K$ randomly sampled normal windows (global baseline), ignoring the current operating regime. This mixes heterogeneous system states and leads to noisy attribution.

\textbf{Conditional retrieval} instead selects $K$ reference samples similar to the \emph{non-anomalous sensors} of the current sample, conditioning the baseline on the current system state.

As shown in Fig.~\ref{fig:conditional_vs_unconditional}, unconditional retrieval produces diffuse heatmaps with multiple false positives, whereas conditional retrieval yields a cleaner attribution with a localized root-cause sensor.

\begin{figure}[!ht]
\centering
\includegraphics[width=0.9\linewidth]{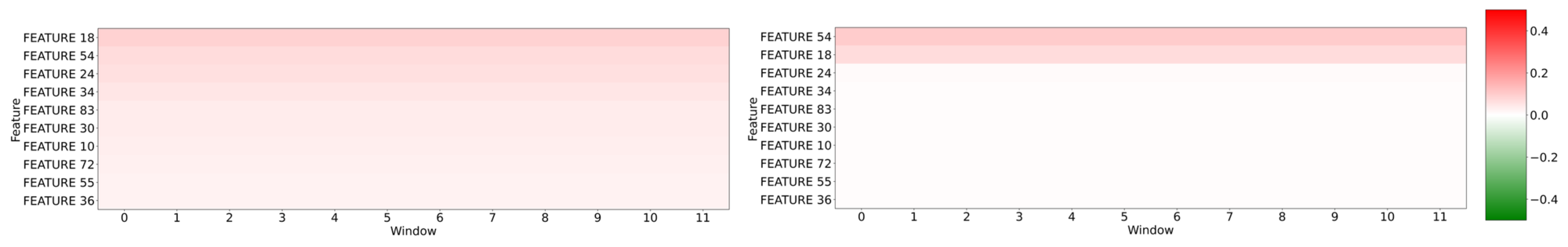}
\caption{\textbf{Unconditional versus conditional retrieval.} Conditional retrieval produces cleaner and more localized attributions by focusing on more relevant reference patterns.}
\label{fig:conditional_vs_unconditional}
\end{figure}

% In supp.tex or paper.tex
\begin{figure*}[!ht]
    \centering

    \begin{subfigure}[t]{0.32\textwidth}
        \centering
        \includegraphics[width=\textwidth]{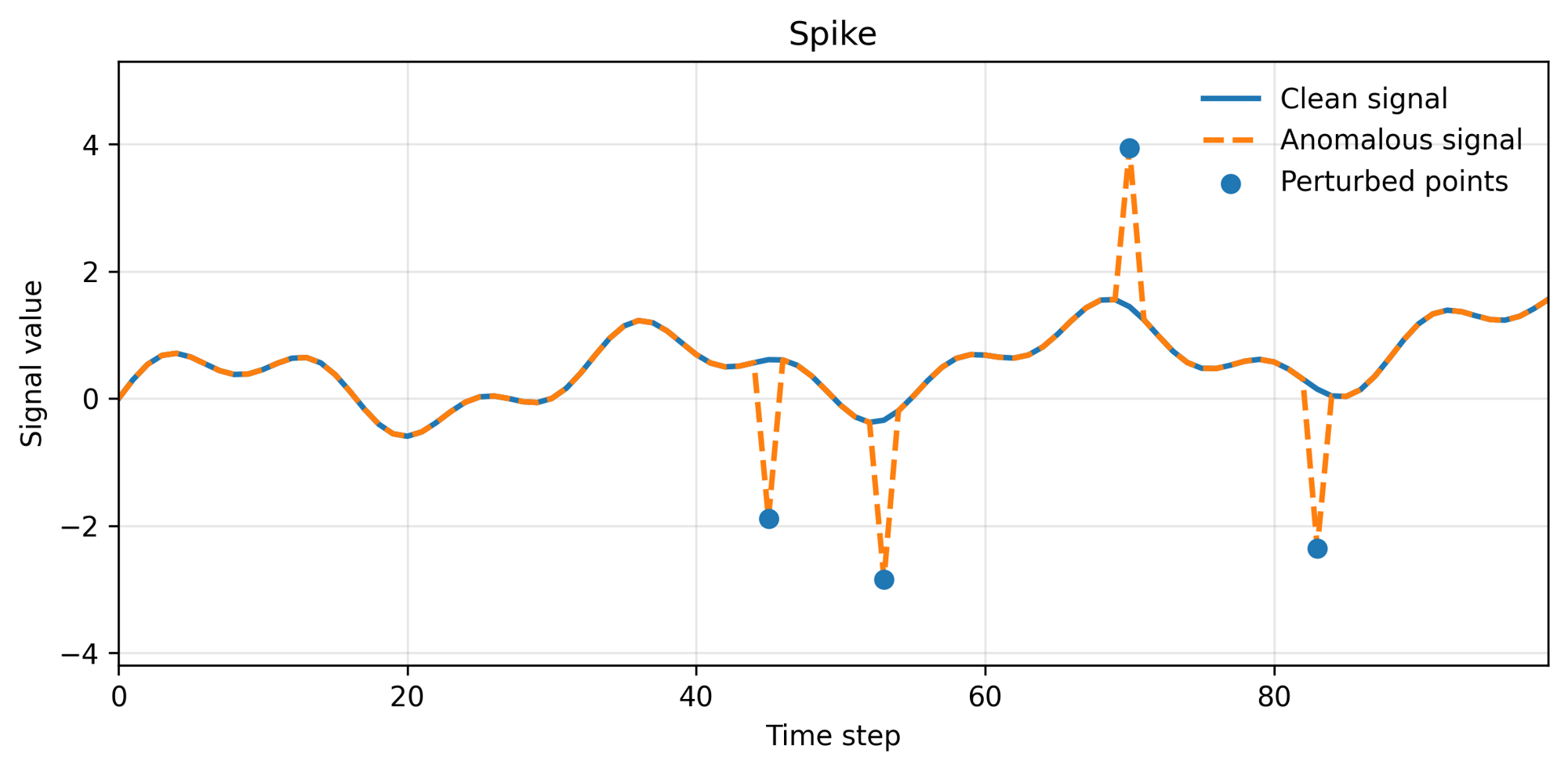}
        \caption{Spike anomaly}
        \label{fig:anomaly_spike}
    \end{subfigure}
    \hfill
    \begin{subfigure}[t]{0.32\textwidth}
        \centering
        \includegraphics[width=\textwidth]{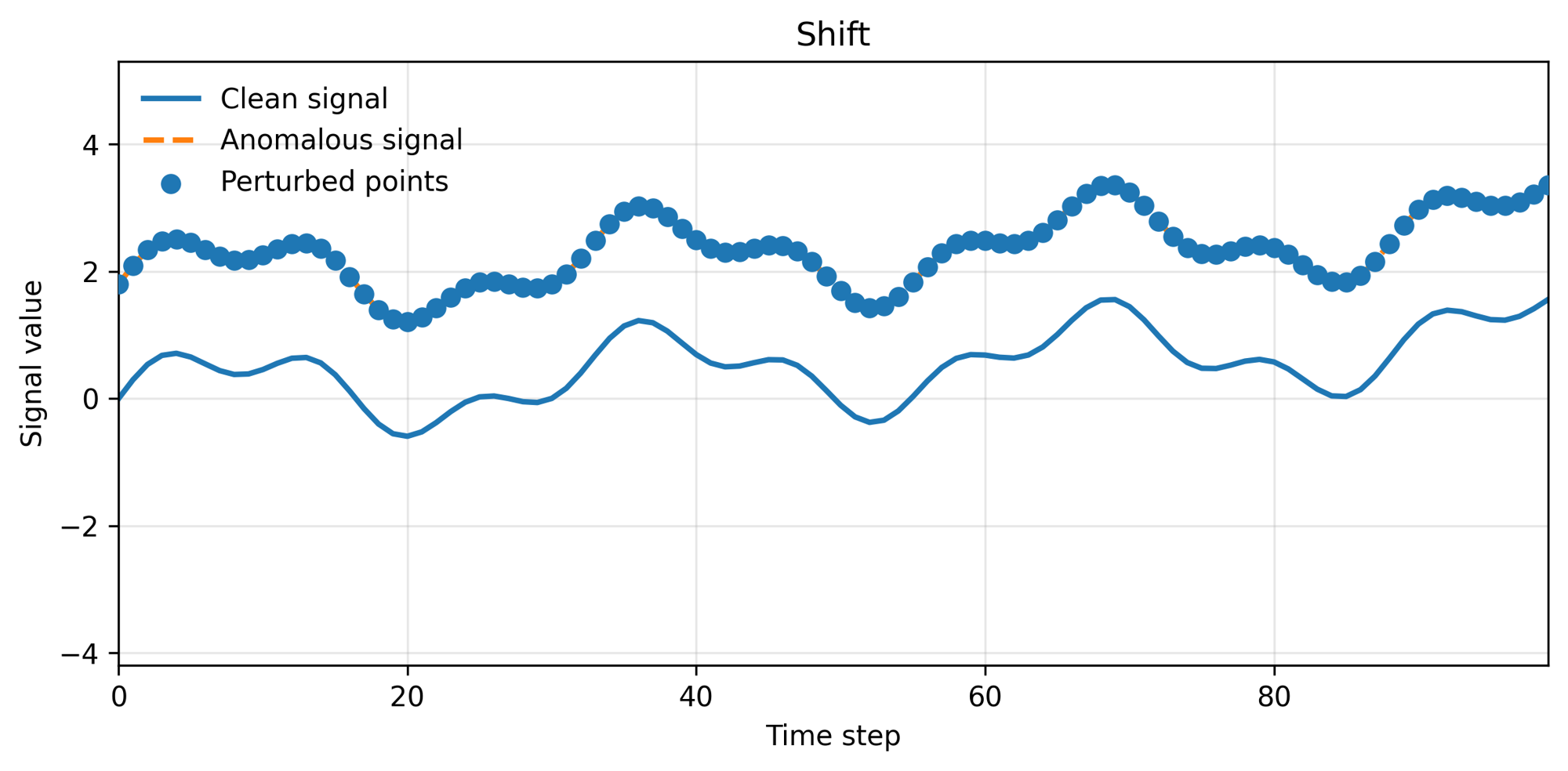}
        \caption{Shift anomaly}
        \label{fig:anomaly_shift}
    \end{subfigure}
    \hfill
    \begin{subfigure}[t]{0.32\textwidth}
        \centering
        \includegraphics[width=\textwidth]{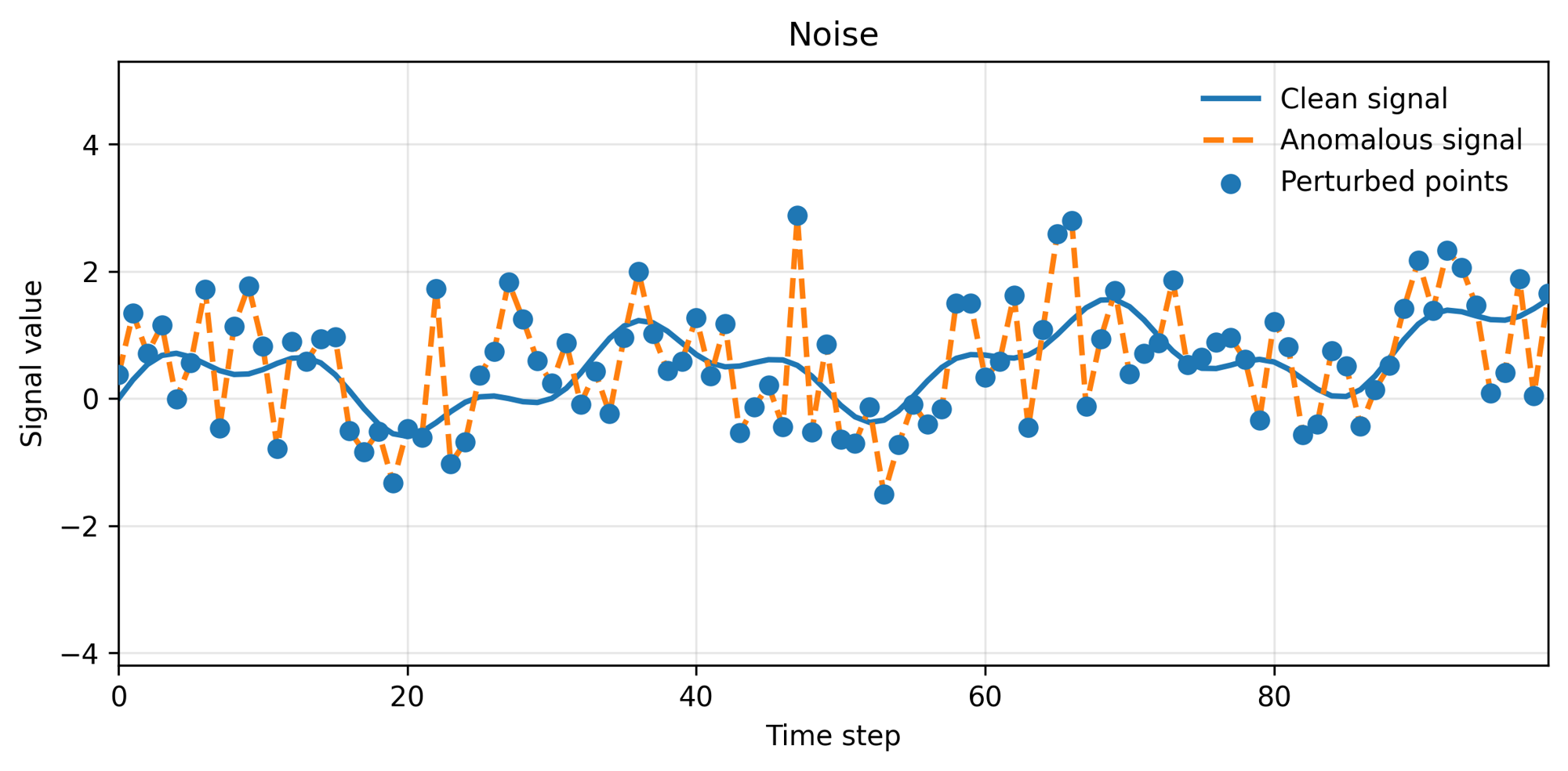}
        \caption{Noise anomaly}
        \label{fig:anomaly_noise}
    \end{subfigure}

    \vspace{0.4em}

    \begin{subfigure}[t]{0.32\textwidth}
        \centering
        \includegraphics[width=\textwidth]{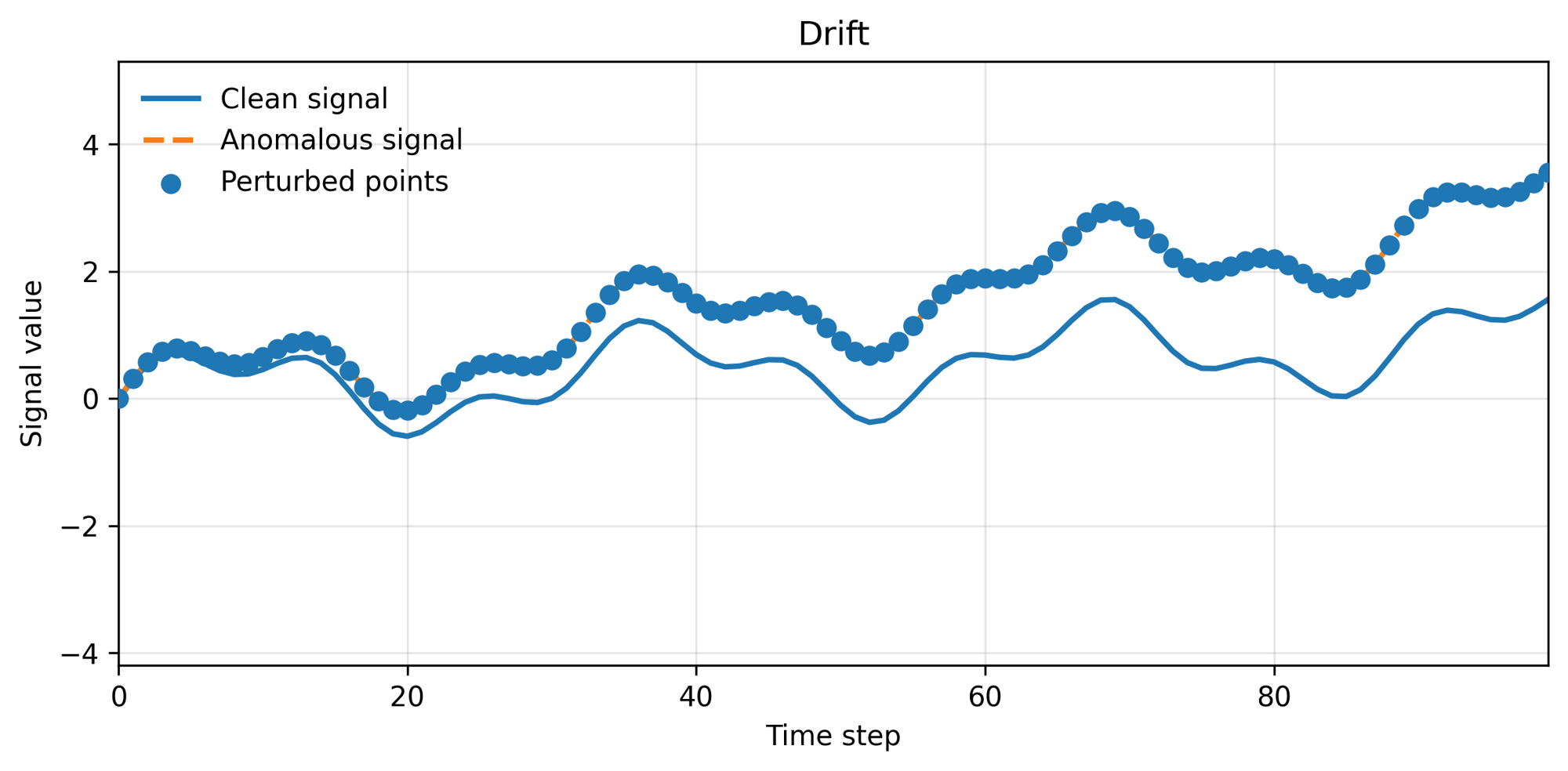}
        \caption{Drift anomaly}
        \label{fig:anomaly_drift}
    \end{subfigure}
    \hfill
    \begin{subfigure}[t]{0.32\textwidth}
        \centering
        \includegraphics[width=\textwidth]{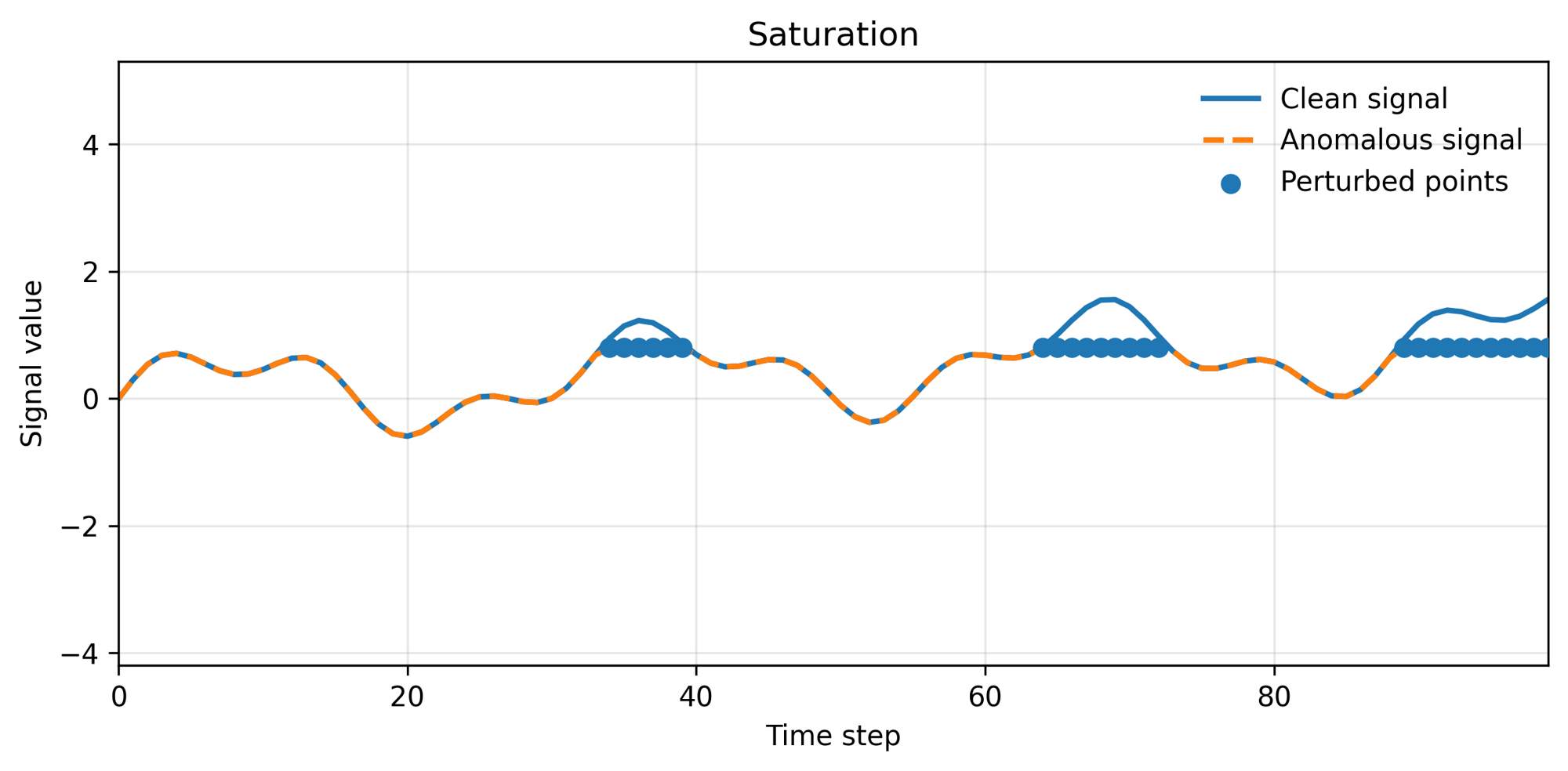}
        \caption{Saturation anomaly}
        \label{fig:anomaly_saturation}
    \end{subfigure}
    \hfill
    \begin{subfigure}[t]{0.32\textwidth}
        \centering
        \includegraphics[width=\textwidth]{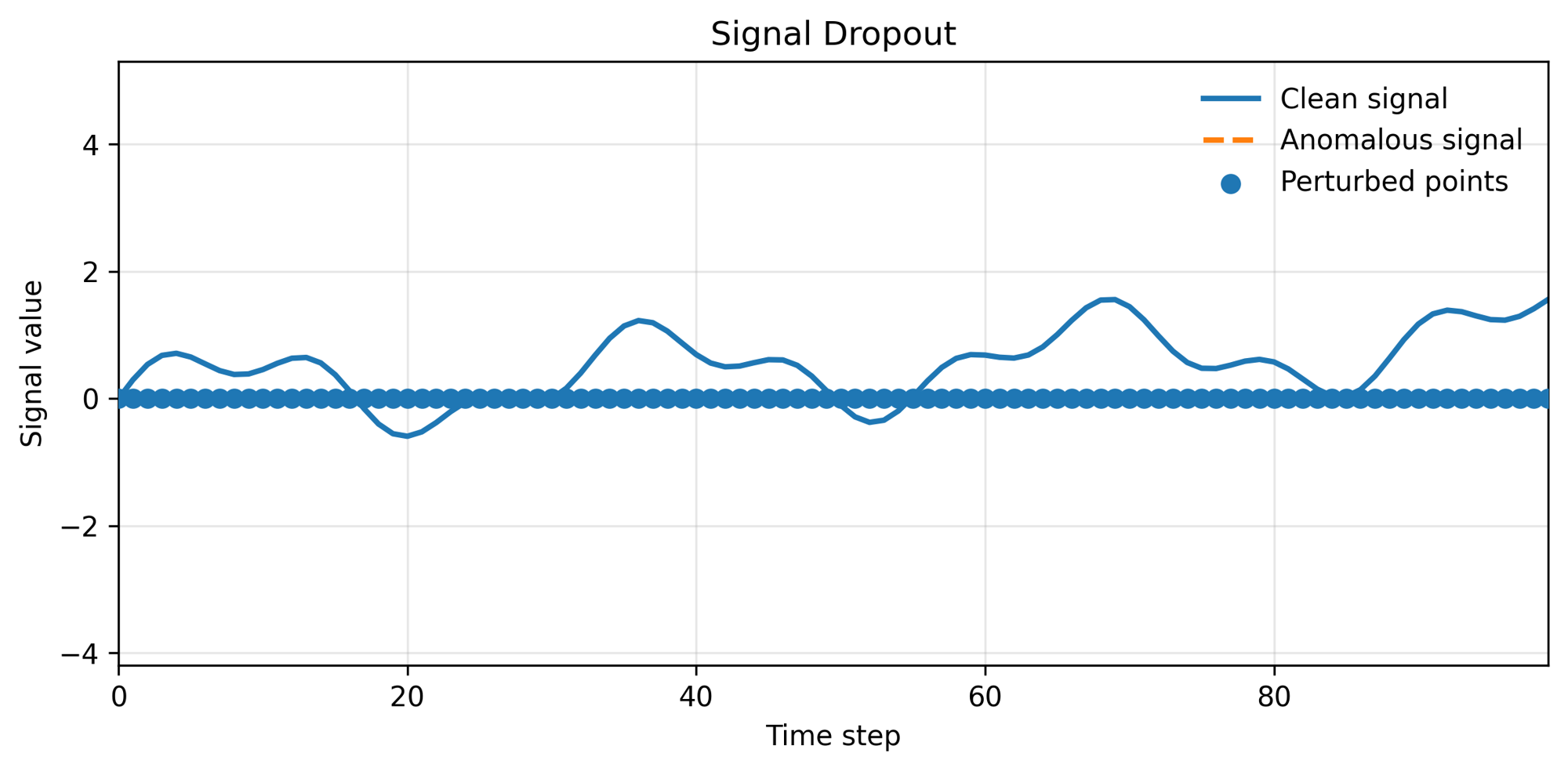}
        \caption{Signal Dropout anomaly}
        \label{fig:anomaly_dropout}
    \end{subfigure}

    \caption{Examples of synthetic anomaly types used in our experiments: spike, shift, noise, drift, saturation, and signal dropout. Each subplot compares the clean signal and the anomalous signal, while highlighting the perturbed points.}
    \label{fig:synthetic_anomaly_types}
\end{figure*}

% \section{Industrial Dataset Details Evaluation}

% This section describes the evaluation setting for the Paul Wurth industrial dataset. To assess the proposed method comprehensively, we consider both controlled and real anomaly scenarios. We first introduce synthetic anomaly injections used to stress-test the method under varying perturbation conditions with known affected sensors. We then present real test samples from the Paul Wurth dataset to evaluate the method on naturally occurring industrial anomalies.

\section{Industrial Dataset Details and Evaluation}

This section describes the evaluation setting for the Paul Wurth industrial dataset. To assess the proposed method comprehensively, we consider both controlled and real anomaly scenarios. We first introduce synthetic anomaly injections used to stress-test the method under varying perturbation conditions with known affected sensors. In addition, expert evaluation confirms the diagnostic utility of the proposed approach. Process engineers verified that the attribution heatmaps accurately captured the progression of thermal instabilities and identified the correct sensor groups in nearly all evaluated anomalies. This qualitative assessment highlights the practical relevance of the framework and supports its suitability for real-world deployment in complex metallurgical monitoring environments.

\subsection{Synthetic Anomaly Injection}

To rigorously evaluate the proposed explainable anomaly detection method, we inject synthetic anomalies into selected sensors of the Paul Wurth industrial time-series data. These perturbations are introduced at different time steps and with varying intensities, enabling a controlled stress test of the method under diverse anomaly conditions. Since the injected anomalies have known locations and affected sensors, they provide ground-truth root causes for assessing whether the proposed approach can accurately localize the underlying sources of anomalous behavior. The considered perturbations are designed to reflect a range of realistic failure patterns commonly encountered in industrial monitoring systems. Alongside the quantitative evaluation, we further provide representative feature plots as qualitative examples, with one plot for each synthetic anomaly type, to visualize the corresponding perturbation patterns. Specifically, we consider the following anomaly types, illustrated in Fig.~\ref{fig:synthetic_anomaly_types}:

\begin{itemize}
    \item \textbf{Spike} anomaly introduces large abnormal values at randomly selected time steps within a window, simulating sudden transient disturbances, impulsive faults, or brief sensor surges.

    \item \textbf{Shift} anomaly adds a constant offset to the entire window, thereby altering the baseline level of the signal throughout the affected segment. This reflects persistent calibration errors or operating-point shifts.

    \item \textbf{Noise} anomaly perturbs the full window with random fluctuations, modeling measurement corruption or stochastic disturbances affecting the signal over the entire interval.

    \item \textbf{Drift} anomaly evolves gradually over time, producing a slow increase or decrease in the signal baseline. This represents sensor aging, calibration drift, or slowly changing process conditions.

    \item \textbf{Signal Dropout} anomaly abruptly sets the entire window to zero, representing sudden sensor failure, communication loss, or complete signal interruption.

    \item \textbf{Saturation} anomaly clips the signal at an upper or lower bound, modeling actuator limits, sensor saturation, or hardware clipping.
\end{itemize}

% Representative feature plots corresponding to each synthetic anomaly type are provided to complement the quantitative evaluation with qualitative analysis.

Representative feature plots illustrating the different synthetic anomaly types are shown in Figures~\ref{fig:spike_anomaly}--\ref{fig:signal_dropout_anomaly}, complementing the quantitative evaluation with qualitative analysis.

\begin{figure*}[t]
    \centering
    \includegraphics[width=0.9\linewidth, height=0.85\textheight]{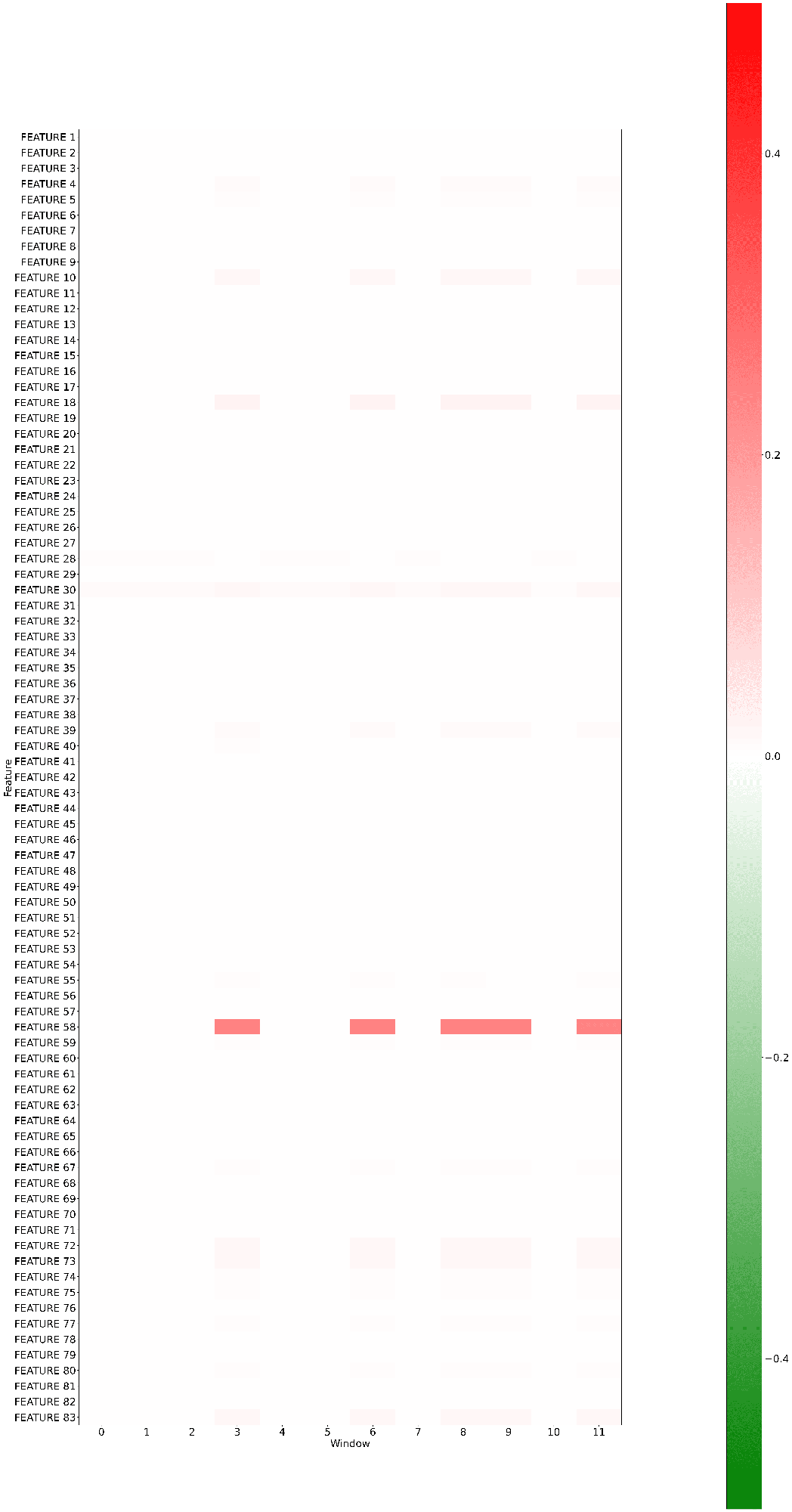}
    \caption{\textbf{Spike anomaly on Feature 58.} The figure shows a representative synthetic spike anomaly, where Feature 58 exhibits large abnormal values at randomly selected time steps, simulating a sudden transient disturbance or brief sensor surge.}
    \label{fig:spike_anomaly}
\end{figure*}
\clearpage

\begin{figure*}[t]
    \centering
    \includegraphics[width=0.9\linewidth, height=0.85\textheight]{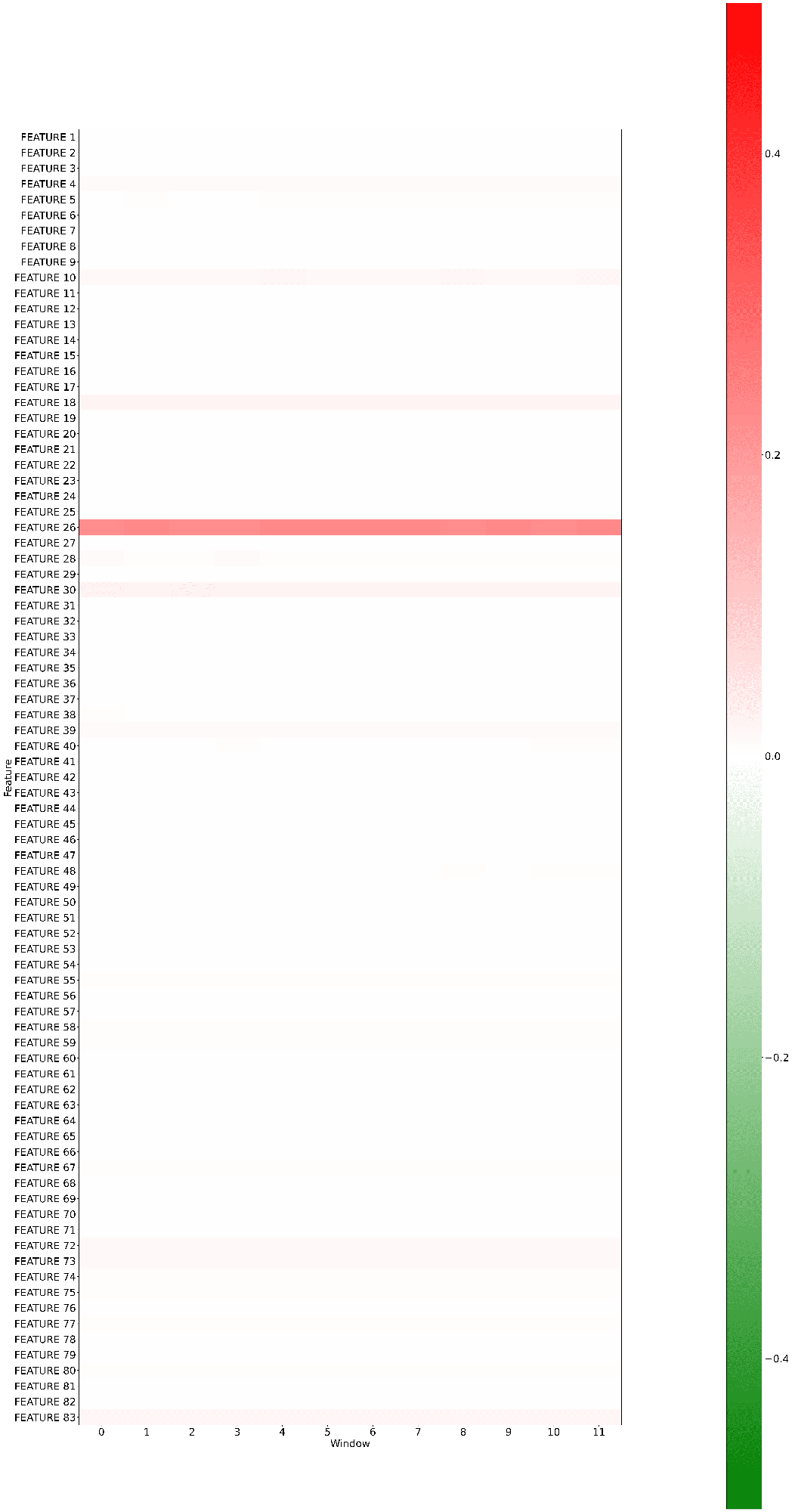}
    \caption{\textbf{Shift anomaly on Feature 26.} Feature 26 is affected by a synthetic shift anomaly, where a constant offset is added across the window, altering the baseline level of the signal.}
    \label{fig:shift_anomaly}
\end{figure*}
\clearpage

\begin{figure*}[t]
    \centering
    \includegraphics[width=0.9\linewidth, height=0.85\textheight]{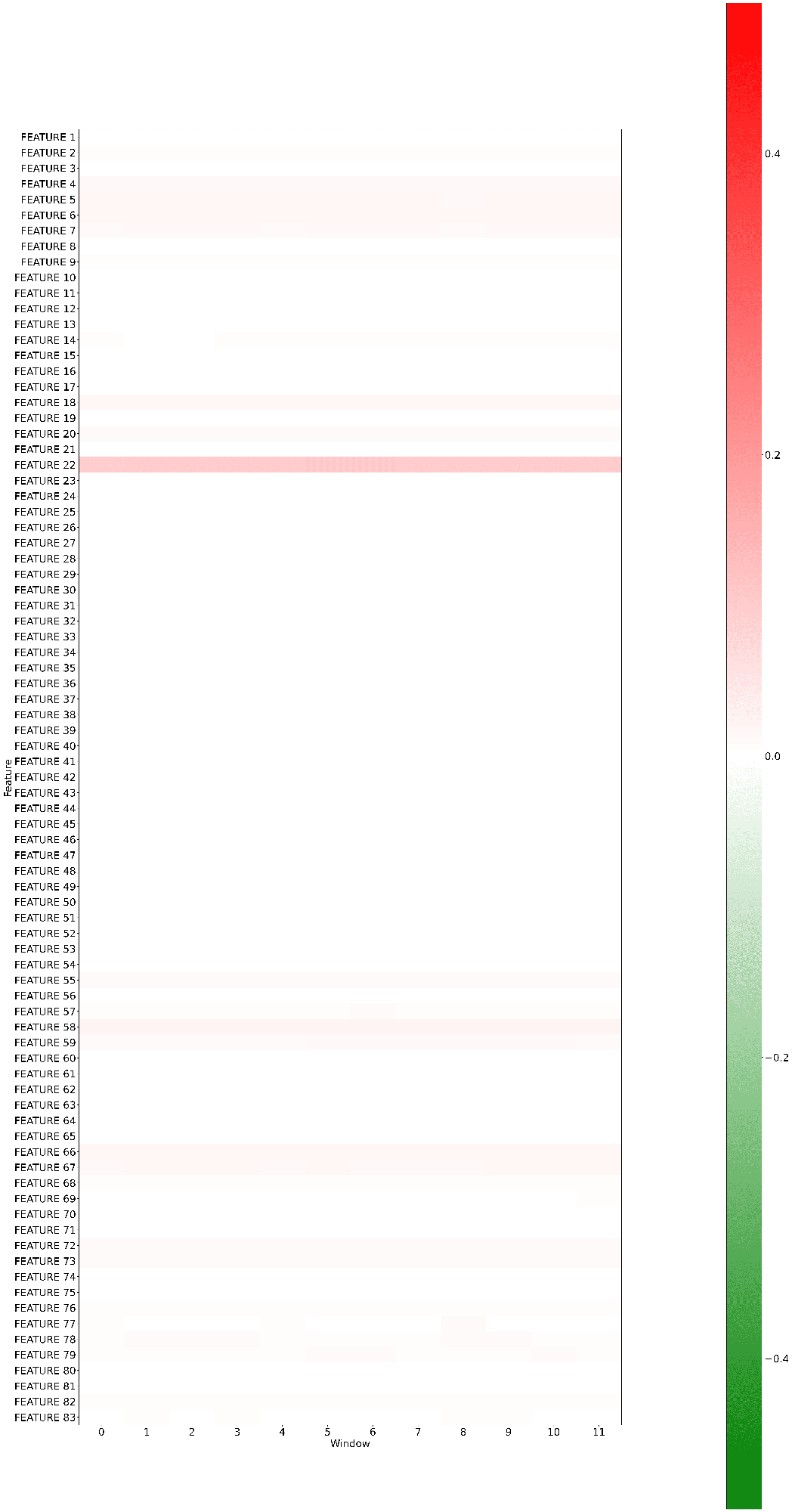}
    \caption{\textbf{Noise anomaly on Feature 22.} Feature 22 is perturbed by a mild synthetic noise anomaly, and the corresponding attribution pattern shows that the anomalous signal is still correctly detected despite the relatively small disturbance.}
    \label{fig:noise_anomaly}
\end{figure*}
\clearpage

\begin{figure*}[t]
    \centering
    \includegraphics[width=0.9\linewidth, height=0.85\textheight]{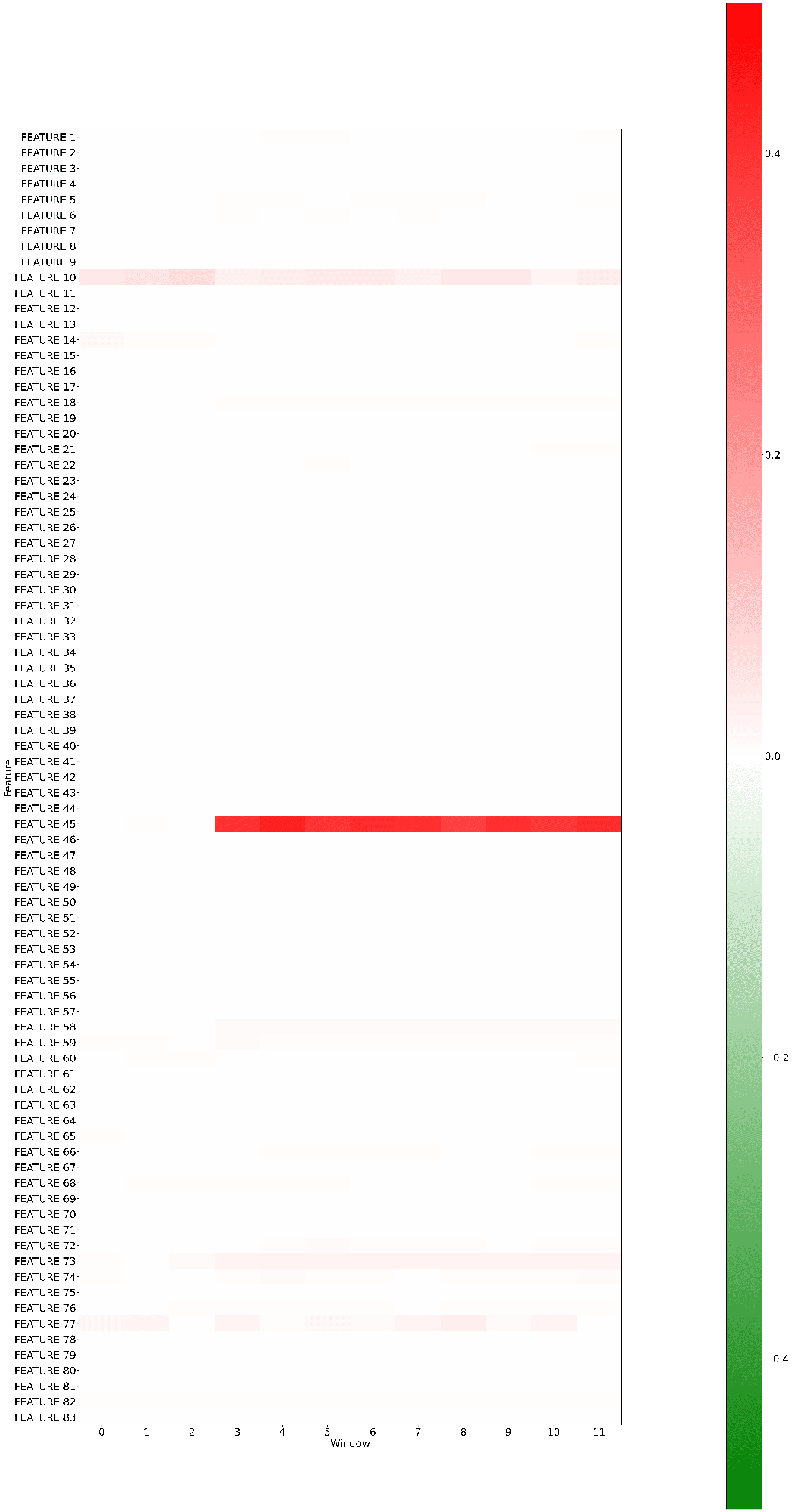}
    \caption{\textbf{Drift anomaly on Feature 45.} Feature 45 exhibits a gradual temporal change corresponding to a synthetic drift anomaly, which is correctly identified by the attribution method.}
    \label{fig:drift_anomaly}
\end{figure*}
\clearpage

\begin{figure*}[t]
    \centering
    \includegraphics[width=0.9\linewidth, height=0.85\textheight]{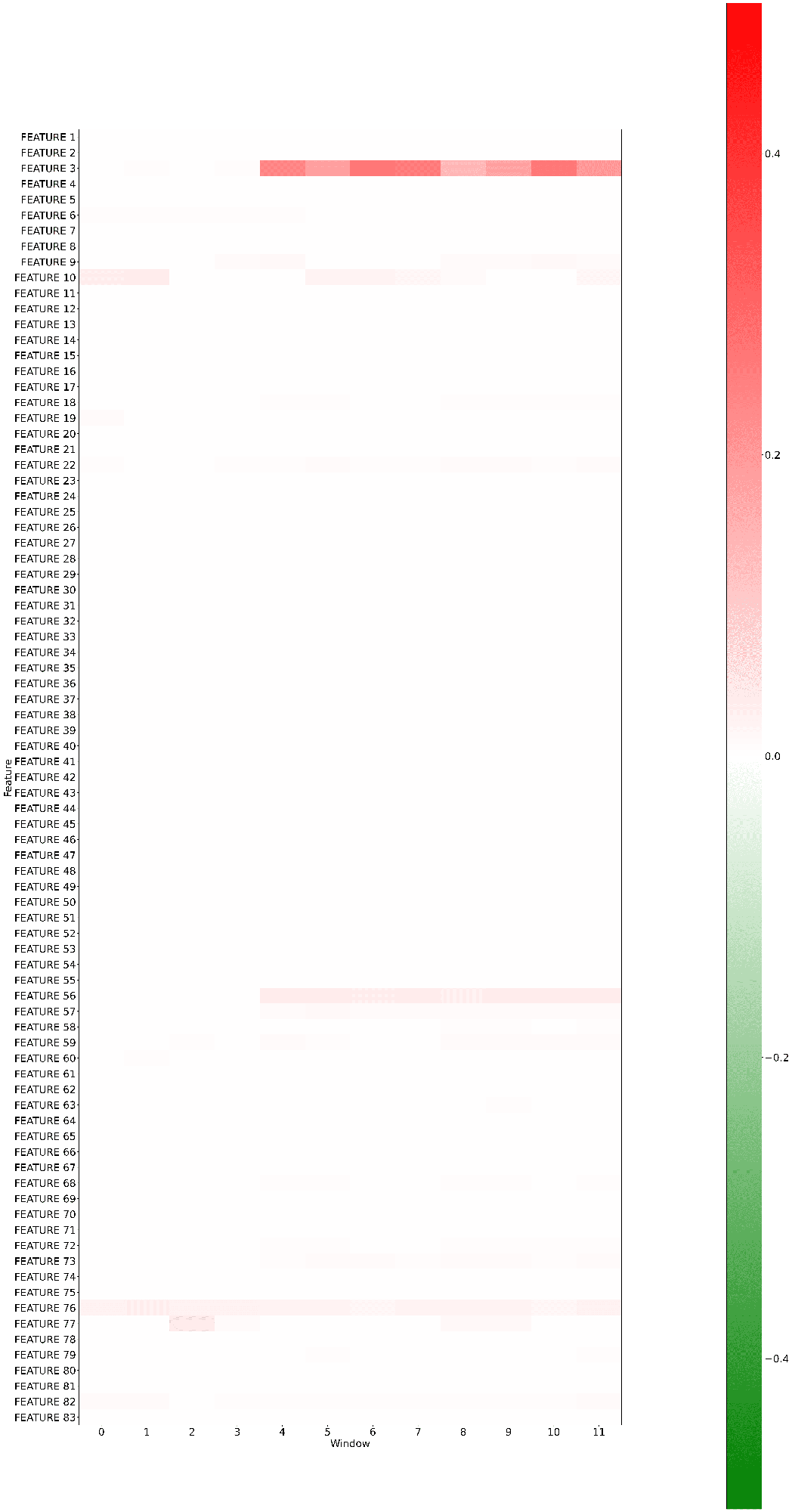}
    \caption{\textbf{Saturation anomaly on Feature 3.} Feature 3 exhibits a synthetic saturation anomaly, in which the signal is clipped at a limit, and the anomaly is correctly identified.}
    \label{fig:saturation_anomaly}
\end{figure*}
\clearpage

\begin{figure*}[t]
    \centering
    \includegraphics[width=0.9\linewidth, height=0.85\textheight]{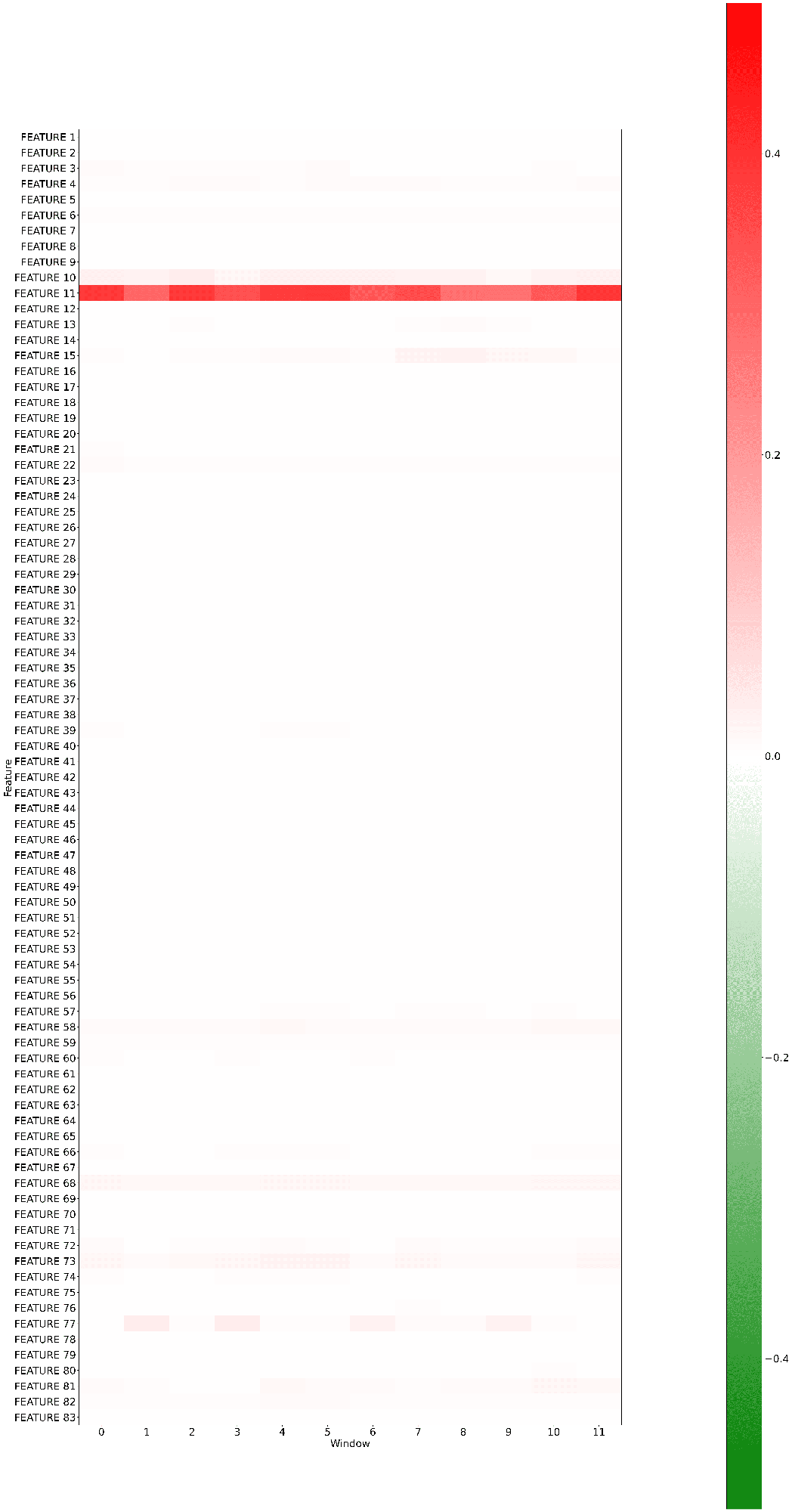}
    \caption{\textbf{Signal dropout anomaly on Feature 11.} Feature 11 is affected by a synthetic signal dropout anomaly, where the signal abruptly drops to zero over the affected interval, and the attribution pattern correctly highlights the anomalous feature.}
    \label{fig:signal_dropout_anomaly}
\end{figure*}
\clearpage

\section{Per-Attack Results on SWaT}

This section reports the complete per-attack results on the SWaT dataset for the considered ablation studies. Tables~\ref{tab:top3_results_blank}--\ref{tab:top10_results_blank} present the results for CondAttr-VAE, while Tables~\ref{tab:top3_results_blank_umap}--\ref{tab:top10_results_blank_umap} present the corresponding results for CondAttr-UMAP.

\begin{table}[!htbp]
\centering
\small
\setlength{\tabcolsep}{4pt}
\renewcommand{\arraystretch}{0.95}
\caption{\textbf{Per-attack performance of CondAttr-VAE on the SWaT dataset.} For each attack instance, we report Top@3R, CW-RCS@3, and TemporalHM@3, alongside the ground-truth sensors correctly identified within the Top-3.}
\begin{tabular}{lccccl}
\toprule
\textbf{Attack \#} & \textbf{Top@3R} & \textbf{CW@3} & \textbf{TempHM@3} & \textbf{FeatureID} \\
\midrule
1  & 1.000 & 0.145 & 0.104 & MV101 \\
2  & 1.000 & 0.121 & 0.545 & P102 \\
3  & 0.000 & 0.000 & 0.307 & --- \\
4  & 0.000 & 0.000 & 0.000 & --- \\
6  & 1.000 & 0.439 & 1.000 & AIT202 \\
7  & 1.000 & 0.888 & 1.000 & LIT301 \\
8  & 1.000 & 0.451 & 0.882 & DPIT301 \\
10 & 1.000 & 0.890 & 1.000 & FIT401 \\
11 & 1.000 & 0.104 & 0.705 & FIT401 \\
13 & 0.000 & 0.000 & 0.000 & --- \\
14 & 0.000 & 0.000 & 0.193 & --- \\
16 & 1.000 & 0.256 & 0.774 & LIT301 \\
17 & 0.000 & 0.000 & 0.000 & --- \\
19 & 0.000 & 0.000 & 0.000 & --- \\
20 & 1.000 & 0.925 & 1.000 & AIT504 \\
21 & 0.500 & 0.136 & 0.617 & LIT101 \\
22 & 0.333 & 0.031 & 0.363 & P501 \\
23 & 0.000 & 0.000 & 0.049 & --- \\
24 & 0.000 & 0.000 & 0.153 & --- \\
25 & 0.500 & 0.386 & 0.628 & LIT401 \\
26 & 0.000 & 0.000 & 0.101 & --- \\
27 & 0.500 & 0.000 & 0.435 & LIT401 \\
28 & 1.000 & 0.093 & 0.711 & P302 \\
29 & 0.000 & 0.000 & 0.000 & --- \\
30 & 0.000 & 0.000 & 0.337 & --- \\
31 & 1.000 & 0.407 & 0.444 & LIT401 \\
32 & 1.000 & 0.790 & 1.000 & LIT301 \\
33 & 1.000 & 0.127 & 0.722 & LIT101 \\
34 & 1.000 & 0.286 & 0.666 & P101 \\
35 & 0.000 & 0.000 & 0.105 & --- \\
36 & 0.000 & 0.000 & 0.716 & --- \\
37 & 0.000 & 0.000 & 0.105 & --- \\
38 & 1.000 & 0.688 & 1.000 & AIT502, AIT402 \\
39 & 0.500 & 0.487 & 0.666 & FIT401 \\
40 & 1.000 & 0.223 & 0.888 & FIT401 \\
41 & 1.000 & 0.946 & 0.876 & LIT301 \\
\bottomrule
\end{tabular}
\label{tab:top3_results_blank}
\end{table}

\begin{table}[t]
\centering
\small
\setlength{\tabcolsep}{4pt}
\renewcommand{\arraystretch}{1.0}
\caption{\textbf{Per-attack performance of CondAttr-VAE on the SWaT dataset.} For each attack instance, we report Top@5R, CW-RCS@5, and TemporalHM@5, alongside the ground-truth sensors correctly identified within the Top-5.}
\begin{tabular}{lccccl}
\toprule
\textbf{Attack \#} & \textbf{Top@5R} & \textbf{CW@5} & \textbf{TempHM@5} & \textbf{FeatureID} \\
\midrule
1  & 1.000 & 0.145 & 0.104 & MV101 \\
2  & 1.000 & 0.121 & 0.666 & P102 \\
3  & 0.000 & 0.000 & 0.566 & --- \\
4  & 0.000 & 0.000 & 0.000 & --- \\
6  & 1.000 & 0.439 & 1.000 & AIT202 \\
7  & 1.000 & 0.888 & 1.000 & LIT301 \\
8  & 1.000 & 0.451 & 1.000 & DPIT301 \\
10 & 1.000 & 0.890 & 1.000 & FIT401 \\
11 & 1.000 & 0.104 & 0.777 & FIT401 \\
13 & 0.000 & 0.000 & 0.000 & --- \\
14 & 0.000 & 0.000 & 0.193 & --- \\
16 & 1.000 & 0.256 & 1.000 & LIT301 \\
17 & 0.000 & 0.000 & 0.000 & --- \\
19 & 0.000 & 0.000 & 0.000 & --- \\
20 & 1.000 & 0.925 & 1.000 & AIT504 \\
21 & 0.500 & 0.136 & 0.666 & LIT101 \\
22 & 0.667 & 0.123 & 0.500 & P501, UV401 \\
23 & 0.333 & 0.002 & 0.049 & --- \\
24 & 0.500 & 0.037 & 0.153 & P203 \\
25 & 0.500 & 0.386 & 0.628 & LIT401 \\
26 & 0.000 & 0.000 & 0.301 & --- \\
27 & 0.500 & 0.000 & 0.462 & LIT401 \\
28 & 1.000 & 0.093 & 0.915 & P302 \\
29 & 0.000 & 0.000 & 0.000 & --- \\
30 & 0.667 & 0.152 & 0.660 & P101, MV201 \\
31 & 1.000 & 0.407 & 0.444 & LIT401 \\
32 & 1.000 & 0.790 & 1.000 & LIT301 \\
33 & 1.000 & 0.127 & 0.866 & LIT101 \\
34 & 1.000 & 0.286 & 0.666 & P101 \\
35 & 0.500 & 0.013 & 0.105 & P101 \\
36 & 0.000 & 0.000 & 0.783 & --- \\
37 & 0.000 & 0.000 & 0.200 & --- \\
38 & 1.000 & 0.688 & 1.000 & AIT502, AIT402 \\
39 & 0.500 & 0.487 & 0.666 & FIT401 \\
40 & 1.000 & 0.223 & 1.000 & FIT401 \\
41 & 1.000 & 0.946 & 0.876 & LIT301 \\
\bottomrule
\end{tabular}
\label{tab:top5_results_blank}
\end{table}

\begin{table}[t]
\centering
\small
\setlength{\tabcolsep}{4pt}
\renewcommand{\arraystretch}{1.0}
\caption{\textbf{Per-attack performance of CondAttr-VAE on the SWaT dataset.} For each attack instance, we report Top@10R, CW-RCS@10, and TemporalHM@10, alongside the ground-truth sensors correctly identified within the Top-10.}
\begin{tabular}{lccccl}
\toprule
\textbf{Attack \#} & \textbf{Top@10R} & \textbf{CW@10} & \textbf{TempHM@10} & \textbf{FeatureID} \\
\midrule
1  & 1.000 & 0.145 & 0.104 & MV101 \\
2  & 1.000 & 0.121 & 0.769 & P102 \\
3  & 1.000 & 0.027 & 0.566 & --- \\
4  & 0.000 & 0.000 & 0.000 & --- \\
6  & 1.000 & 0.439 & 1.000 & AIT202 \\
7  & 1.000 & 0.888 & 1.000 & LIT301 \\
8  & 1.000 & 0.451 & 1.000 & DPIT301 \\
10 & 1.000 & 0.890 & 1.000 & FIT401 \\
11 & 1.000 & 0.104 & 1.000 & FIT401 \\
13 & 0.000 & 0.000 & 0.000 & --- \\
14 & 0.000 & 0.000 & 0.193 & --- \\
16 & 1.000 & 0.256 & 1.000 & LIT301 \\
19 & 1.000 & 0.010 & 1.000 & --- \\
20 & 1.000 & 0.925 & 1.000 & AIT504 \\
21 & 0.500 & 0.136 & 0.666 & LIT101 \\
22 & 0.667 & 0.123 & 0.540 & P501, UV401 \\
23 & 0.333 & 0.002 & 0.049 & --- \\
24 & 0.500 & 0.037 & 0.400 & P203 \\
25 & 0.500 & 0.386 & 0.628 & LIT401 \\
26 & 0.500 & 0.036 & 0.564 & --- \\
27 & 0.500 & 0.000 & 0.462 & LIT401 \\
28 & 1.000 & 0.093 & 0.976 & P302 \\
29 & 0.000 & 0.000 & 0.000 & --- \\
30 & 1.000 & 0.240 & 0.946 & P101, MV201 \\
31 & 1.000 & 0.407 & 0.444 & LIT401 \\
32 & 1.000 & 0.790 & 1.000 & LIT301 \\
33 & 1.000 & 0.127 & 0.866 & LIT101 \\
34 & 1.000 & 0.286 & 1.000 & P101 \\
35 & 0.500 & 0.013 & 0.285 & P101 \\
36 & 0.000 & 0.000 & 0.783 & --- \\
37 & 0.500 & 0.011 & 0.285 & P501 \\
38 & 1.000 & 0.688 & 1.000 & AIT502, AIT402 \\
39 & 0.500 & 0.487 & 0.769 & FIT401 \\
40 & 1.000 & 0.223 & 1.000 & FIT401 \\
41 & 1.000 & 0.946 & 0.876 & LIT301 \\
\bottomrule
\end{tabular}
\label{tab:top10_results_blank}
\end{table}

\begin{table}[t]
\centering
\small
\setlength{\tabcolsep}{4pt}
\renewcommand{\arraystretch}{1.0}
\caption{\textbf{Per-attack performance of CondAttr-UMAP on the SWaT dataset.} For each attack instance, we report Top@3R, CW-RCS@3, and TemporalHM@3, alongside the ground-truth sensors correctly identified within the Top-3.}
\begin{tabular}{lccccl}
\toprule
\textbf{Attack \#} & \textbf{Top@3R} & \textbf{CW@3} & \textbf{TempHM@3} & \textbf{FeatureID} \\
\midrule
1  & 0.000 & 0.000 & 0.000 & --- \\
2  & 1.000 & 0.300 & 0.933 & P102 \\
3  & 0.000 & 0.000 & 0.600 & --- \\
4  & 0.000 & 0.000 & 0.000 & --- \\
6  & 1.000 & 0.443 & 1.000 & AIT202 \\
7  & 1.000 & 0.911 & 1.000 & LIT301 \\
8  & 1.000 & 0.375 & 0.882 & DPIT301 \\
10 & 1.000 & 0.851 & 1.000 & FIT401 \\
11 & 1.000 & 0.118 & 0.533 & FIT401 \\
13 & 0.000 & 0.000 & 0.000 & --- \\
14 & 0.000 & 0.000 & 0.000 & --- \\
16 & 0.000 & 0.000 & 0.774 & --- \\
17 & 0.000 & 0.000 & 0.000 & --- \\
19 & 0.000 & 0.000 & 0.000 & --- \\
20 & 1.000 & 0.999 & 1.000 & AIT504 \\
21 & 0.500 & 0.151 & 0.666 & LIT101 \\
22 & 0.667 & 0.217 & 0.540 & UV401, P501 \\
23 & 0.333 & 0.169 & 0.500 & MV302 \\
24 & 0.500 & 0.068 & 0.153 & P205 \\
25 & 0.500 & 0.088 & 0.545 & LIT401 \\
26 & 0.500 & 0.019 & 0.403 & LIT301 \\
27 & 0.500 & 0.000 & 0.420 & LIT401 \\
28 & 0.000 & 0.000 & 0.000 & --- \\
29 & 0.000 & 0.000 & 0.000 & --- \\
30 & 0.333 & 0.002 & 0.180 & P101 \\
31 & 1.000 & 0.407 & 0.923 & LIT401 \\
32 & 1.000 & 0.931 & 0.912 & LIT301 \\
33 & 0.000 & 0.000 & 0.452 & --- \\
34 & 1.000 & 0.342 & 0.666 & P101 \\
35 & 0.000 & 0.000 & 0.105 & --- \\
36 & 0.000 & 0.000 & 0.716 & --- \\
37 & 0.000 & 0.000 & 0.000 & --- \\
38 & 1.000 & 0.803 & 1.000 & AIT502, AIT402 \\
39 & 0.500 & 0.439 & 0.720 & FIT401 \\
40 & 1.000 & 0.167 & 0.571 & FIT401 \\
41 & 1.000 & 0.956 & 0.907 & LIT301 \\
\bottomrule
\end{tabular}
\label{tab:top3_results_blank_umap}
\end{table}

\begin{table}[t]
\centering
\small
\setlength{\tabcolsep}{4pt}
\renewcommand{\arraystretch}{1.0}
\caption{\textbf{Per-attack performance of CondAttr-UMAP on the SWaT dataset.} For each attack instance, we report Top@5R, CW-RCS@5, and TemporalHM@5, alongside the ground-truth sensors correctly identified within the Top-5.}
\begin{tabular}{lccccl}
\toprule
\textbf{Attack \#} & \textbf{Top@5R} & \textbf{CW@5} & \textbf{TempHM@5} & \textbf{FeatureID} \\
\midrule
1  & 0.000 & 0.000 & 0.000 & --- \\
2  & 1.000 & 0.300 & 1.000 & P102 \\
3  & 0.000 & 0.000 & 0.600 & --- \\
4  & 0.000 & 0.000 & 0.000 & --- \\
6  & 1.000 & 0.443 & 1.000 & AIT202 \\
7  & 1.000 & 0.911 & 1.000 & LIT301 \\
8  & 1.000 & 0.375 & 1.000 & DPIT301 \\
10 & 1.000 & 0.851 & 1.000 & FIT401 \\
11 & 1.000 & 0.118 & 0.900 & FIT401 \\
13 & 0.000 & 0.000 & 0.000 & --- \\
14 & 0.000 & 0.000 & 0.000 & --- \\
16 & 1.000 & 0.063 & 0.774 & LIT301 \\
17 & 0.000 & 0.000 & 0.000 & MV303  \\
19 & 0.000 & 0.000 & 0.000 & AIT504  \\
20 & 1.000 & 0.999 & 1.000 & AIT504 \\
21 & 0.500 & 0.151 & 0.666 & LIT101 \\
22 & 0.667 & 0.217 & 0.744 & UV401, P501 \\
23 & 0.333 & 0.169 & 0.500 & MV302,  DPIT301\\
24 & 1.000 & 0.245 & 0.285 & P205, P203 \\
25 & 0.500 & 0.088 & 0.545 & LIT401 \\
26 & 0.500 & 0.019 & 0.501 & LIT301 \\
27 & 0.500 & 0.000 & 0.462 & LIT401 \\
28 & 0.000 & 0.000 & 0.000 & --- \\
29 & 0.000 & 0.000 & 0.000 & --- \\
30 & 0.333 & 0.002 & 0.266 & P101, LIT101 \\
31 & 1.000 & 0.407 & 0.923 & LIT401 \\
32 & 1.000 & 0.931 & 1.000 & LIT301 \\
33 & 0.000 & 0.000 & 0.731 & --- \\
34 & 1.000 & 0.342 & 1.000 & P101 \\
35 & 0.000 & 0.000 & 0.105 & --- \\
36 & 0.000 & 0.000 & 0.716 & --- \\
37 & 0.000 & 0.000 & 0.104 & --- \\
38 & 1.000 & 0.803 & 1.000 & AIT502, AIT402 \\
39 & 0.500 & 0.439 & 0.720 & FIT401 \\
40 & 1.000 & 0.167 & 1.000 & FIT401 \\
41 & 1.000 & 0.956 & 0.907 & LIT301 \\
\bottomrule
\end{tabular}
\label{tab:top5_results_blank_umap}
\end{table}

\begin{table}[t]
\centering
\small
\setlength{\tabcolsep}{4pt}
\renewcommand{\arraystretch}{1.0}
\caption{\textbf{Per-attack performance of CondAttr-UMAP on the SWaT dataset.} For each attack instance, we report Top@10R, CW-RCS@10, and TemporalHM@10, alongside the ground-truth sensors correctly identified within the Top-10.}
\begin{tabular}{lccccl}
\toprule
\textbf{Attack \#} & \textbf{Top@10R} & \textbf{CW@10} & \textbf{TempHM@10} & \textbf{FeatureID} \\
\midrule
1  & 0.000 & 0.000 & 0.000 & --- \\
2  & 1.000 & 0.300 & 1.000 & P102 \\
3  & 0.000 & 0.000 & 0.600 & --- \\
4  & 0.000 & 0.000 & 0.000 & --- \\
6  & 1.000 & 0.443 & 1.000 & AIT202 \\
7  & 1.000 & 0.911 & 1.000 & LIT301 \\
8  & 1.000 & 0.375 & 1.000 & DPIT301 \\
10 & 1.000 & 0.851 & 1.000 & FIT401 \\
11 & 1.000 & 0.118 & 1.000 & FIT401 \\
13 & 0.000 & 0.000 & 0.000 & --- \\
14 & 0.000 & 0.000 & 0.000 & --- \\
16 & 1.000 & 0.063 & 1.000 & LIT301 \\
17 & 1.000 & 0.000 & 0.000 & --- \\
19 & 1.000 & 0.008 & 0.774 & --- \\
20 & 1.000 & 0.999 & 1.000 & AIT504 \\
21 & 0.500 & 0.151 & 0.666 & LIT101 \\
22 & 0.667 & 0.217 & 0.875 & UV401, P501 \\
23 & 0.667 & 0.339 & 0.528 & MV302 \\
24 & 1.000 & 0.245 & 0.285 & P205, P203 \\
25 & 0.500 & 0.088 & 0.628 & LIT401 \\
26 & 0.500 & 0.019 & 0.634 & LIT301 \\
27 & 0.500 & 0.000 & 0.462 & LIT401 \\
28 & 0.000 & 0.000 & 0.002 & --- \\
29 & 0.000 & 0.000 & 0.000 & --- \\
30 & 0.667 & 0.006 & 0.458 & P101 \\
31 & 1.000 & 0.407 & 0.923 & LIT401 \\
32 & 1.000 & 0.931 & 1.000 & LIT301 \\
33 & 1.000 & 0.073 & 1.000 & --- \\
34 & 1.000 & 0.342 & 1.000 & P101 \\
35 & 0.500 & 0.011 & 0.285 & P101 \\
36 & 0.000 & 0.000 & 0.716 & --- \\
37 & 0.000 & 0.000 & 0.200 & --- \\
38 & 1.000 & 0.803 & 1.000 & AIT502, AIT402 \\
39 & 0.500 & 0.439 & 0.720 & FIT401 \\
40 & 1.000 & 0.167 & 1.000 & FIT401 \\
41 & 1.000 & 0.956 & 0.907 & LIT301 \\
\bottomrule
\end{tabular}
\label{tab:top10_results_blank_umap}
\end{table}
%\end{document}

\end{document}